\documentclass[orivec]{llncs}

\usepackage{amssymb, amsmath, enumerate, mdwlist, fancyhdr}
\usepackage[latin1]{inputenc}
\usepackage[T1]{fontenc}
\usepackage{ifthen}
\usepackage{xspace}
\usepackage[ruled,vlined,linesnumbered]{algorithm2e}
\SetKw{To}{to}
\newcommand{\assign}{\leftarrow}
\usepackage{tikz}
\usetikzlibrary{shapes,fit, decorations.pathreplacing}

\usepackage{xspace}
\providecommand{\ignore}[1]{}
\newcommand{\ce}{\texttt{ce}\xspace}
\newcommand{\ORT}{\textbf{ORT}\xspace}

\newcommand{\Text}[1][]{\mathbf{Txt}^{#1}}
\newcommand{\Gold}{\mathbf{G}}
\newcommand{\Sd}{\mathbf{Sd}}
\newcommand{\Psd}{\mathbf{Psd}}
\newcommand{\Ex}{\mathbf{Ex}}
\newcommand{\Conv}{\mathbf{Conv}}
\newcommand{\Caut}{\mathbf{Caut}}
\newcommand{\NU}{\mathbf{NU}}
\newcommand{\SNU}{\mathbf{SNU}}
\newcommand{\SynDec}{\mathbf{SynDec}}
\newcommand{\Dec}{\mathbf{Dec}}
\newcommand{\SDec}{\mathbf{SDec}}
\newcommand{\WMon}{\mathbf{WMon}}
\newcommand{\Mon}{\mathbf{Mon}}
\newcommand{\SMon}{\mathbf{SMon}}

\newcommand{\natnum}{\mathbb{N}}
\newcommand{\seq}{{\mathbb{S}\mathrm{eq}}}

\newcommand{\set}[2]{\{#1 \; | \; #2\}}
\newcommand{\pad}{\mathrm{pad}}
\newcommand{\dom}{\mathrm{dom}}
\newcommand{\range}{\mathrm{range}}
\newcommand{\content}{\mathrm{content}}
\newcommand{\ind}{\mathrm{ind}}
\newcommand{\converges}{\mathclose{\hbox{$\downarrow$}}}
\newcommand{\diverges}{\mathclose{\hbox{$\uparrow$}}}
\newcommand{\sqr}[2]{{\vcenter{\hrule height.#2pt
        \hbox{\vrule width.#2pt height#1pt \kern#1pt
           \vrule width.#2pt}
        \hrule height.#2pt}}}
\newcommand{\Qed}[1]{{\nobreak\hfil\penalty50
  \hskip1em\hbox{}\nobreak\hfil$\sqr{8}{8}$\enspace\sc #1
  \parfillskip=0pt \finalhyphendemerits=0 \par}
  \bigskip}
\newcommand{\QedFor}[1]{\Qed{(for~#1)}}
\newcommand{\QedForClaim}[1][]{\QedFor{~Claim#1}}
\renewenvironment{proof}{
\noindent
\noindent\emph{Proof.}\enspace}{\Qed{}}
\newcommand{\CalC}{\mathcal{C}}
\newcommand{\CalE}{\mathcal{E}}
\newcommand{\CalL}{\mathcal{L}}
\newcommand{\CalP}{\mathcal{P}}
\newcommand{\CalR}{\mathcal{R}}
\newcommand{\FrakP}{\mathfrak{P}}
\newcommand{\FrakR}{\mathfrak{R}}
\spnewtheorem{thm}{Theorem}{\bf }{\rm }
\spnewtheorem{prop}[thm]{Proposition}{\bf }{\rm }
\spnewtheorem{cor}[thm]{Corollary}{\bf }{\rm }
\spnewtheorem{lem}[thm]{Lemma}{\bf }{\rm }
\spnewtheorem{defn}[thm]{Definition}{\bf }{\rm }
\spnewtheorem{rem}[thm]{Remark}{\bf }{\rm }
\spnewtheorem{asm}[thm]{Assumption}{\bf }{\rm }
\spnewtheorem{notation}[thm]{Notation}{\bf }{\rm }
\spnewtheorem{ex}[thm]{Example}{\bf }{\rm }
\spnewtheorem{qu}[thm]{Question}{\bf }{\rm }
\spnewtheorem{tobedone}[thm]{To Be Done}{\bf }{\rm }

\newcounter{claimCounter}[thm]
\renewenvironment{claim}{
\vspace{3mm}\noindent\refstepcounter{claimCounter}\emph{Claim~\theclaimCounter.}\enspace}{}
\newenvironment{claimProof}{
\noindent
\emph{Proof of Claim~\theclaimCounter.}\enspace}{\QedForClaim[~\theclaimCounter]}

\usepackage{ifthen}
\newboolean{useproof}
\setboolean{useproof}{false}
\def\makeinnocent#1{\catcode`#1=12 }
\def\csarg#1#2{\expandafter#1\csname#2\endcsname}
\def\ThrowAwayComment#1{\begingroup
    \def\CurrentComment{#1}    \let\do\makeinnocent \dospecials
    \makeinnocent\^^L    \endlinechar`\^^M \catcode`\^^M=12 \xComment}
{\catcode`\^^M=12 \endlinechar=-1  \gdef\xComment#1^^M{\def\test{#1}
      \csarg\ifx{PlainEnd\CurrentComment Test}\test
          \let\next\endgroup
      \else \csarg\ifx{LaLaEnd\CurrentComment Test}\test
            \edef\next{\endgroup\noexpand\end{\CurrentComment}}
      \else \let\next\xComment
      \fi \fi \next}
}
{\escapechar=-1\relax
	\csarg\xdef{PlainEndproofTest}{\string\\endproof}	\csarg\xdef{LaLaEndproofTest}{\string\\end\string\{proof\string\}}}
\def \proof {\ifthenelse{\boolean{useproof}}{
\noindent
\noindent\emph{Proof.}\enspace
}{\ThrowAwayComment{proof}}}
\def \endproof {\ifthenelse{\boolean{useproof}}{\Qed{}}{}}

\title{A Map of Update Constraints\\ in Inductive Inference}
\author{Timo K\"otzing \and Raphaela Palenta}
\institute{
Friedrich-Schiller-Universit{\"a}t Jena, Germany\\
\email{\{timo.koetzing,raphaela-julia.palenta\}@uni-jena.de}
}
\begin{document}

\newcommand{\SemRestr}{\mathrm{Sem}}
\newcommand{\McRestr}{\mathrm{Mc}}

\newcommand{\id}{\mathrm{id}}
\newcommand{\findWitness}{\texttt{findWitness}}

\newcommand{\emptysequence}{\lambda}

\maketitle

\setboolean{useproof}{true}

\begin{abstract}
We investigate how different learning restrictions reduce learning power and how the different restrictions relate to one another. We give a complete map for nine different restrictions both for the cases of complete information learning and set-driven learning. This completes the picture for these well-studied \emph{delayable} learning restrictions. A further insight is gained by different characterizations of \emph{conservative} learning in terms of variants of \emph{cautious} learning.

Our analyses greatly benefit from general theorems we give, for example showing that learners with exclusively delayable restrictions can always be assumed total.
\end{abstract}

\section{Introduction}
\setboolean{useproof}{true}

\label{sec:Introduction}

This paper is set in the framework of \emph{inductive inference}, a branch of (algorithmic) learning theory. This branch analyzes the problem of algorithmically learning a description for a formal language (a computably enumerable subset of the set of natural numbers) when presented successively all and only the elements of that language. For example, a learner $h$ might be presented more and more even numbers. After each new number, $h$ outputs a description for a language as its conjecture. The learner $h$ might decide to output a program for the set of all multiples of $4$, as long as all numbers presented are divisible by~$4$. Later, when $h$ sees an even number not divisible by $4$, it might change this guess to a program for the set of all multiples of~$2$.

Many criteria for deciding whether a learner $h$ is \emph{successful} on a language~$L$ have been proposed in the literature. Gold, in his seminal paper \cite{Gol:j:67}, gave a first,
simple learning criterion, \emph{$\Text\Gold\Ex$-learning}\footnote{$\Text$ stands for learning from a \emph{text} of positive examples; $\Gold$ stands for Gold, who introduced this model, and is used to to indicate full-information learning; $\Ex$ stands for \emph{explanatory}.}, where a learner is \emph{successful} iff, on every \emph{text} for $L$ (listing of all and only the elements of $L$) it eventually stops changing its conjectures, and its final conjecture is a correct description for the input sequence.
Trivially, each single, describable language $L$ has a suitable constant function as a $\Text\Gold\Ex$-learner (this learner constantly outputs a description for $L$). Thus, we are interested in analyzing for which \emph{classes of languages} $\CalL$ there is a \emph{single learner} $h$ learning \emph{each} member of $\CalL$. This framework is also sometimes known as \emph{language learning in the limit} and has been studied extensively, using a wide range of learning criteria similar to $\Text\Gold\Ex$-learning (see, for example, the textbook \cite{Jai-Osh-Roy-Sha:b:99:stl2}).

A wealth of learning criteria can be derived from $\Text\Gold\Ex$-learning by adding restrictions on the intermediate conjectures and how they should relate to each other and the data. For example, one could require that a conjecture which is consistent with the data must not be changed; this is known as \emph{conservative} learning and known to restrict what classes of languages can be learned (\cite{Ang:j:80:lang-pos-data}, we use $\Conv$ to denote the restriction of conservative learning). Additionally to conservative learning, the following learning restrictions are considered in this paper (see Section~\ref{sec:LearningCriteria} for a formal definition of learning criteria including these learning restrictions).

In \emph{cautious} learning ($\Caut$, \cite{Osh-Sto-Wei:j:82:strategies}) the learner is not allowed to ever give a conjecture for a strict subset of a previously conjectured set. In \emph{non-U-shaped} learning ($\NU$, \cite{Bal-Cas-Mer-Ste-Wie:j:08}) a learner may never \emph{semantically} abandon a correct conjecture; in \emph{strongly non-U-shaped} learning ($\SNU$, \cite{Cas-Moe:j:11:optLan}) not even syntactic changes are allowed after giving a correct conjecture.

In \emph{decisive} learning ($\Dec$, \cite{Osh-Sto-Wei:j:82:strategies}), a learner may never (semantically) return to a \emph{semantically} abandoned conjecture; in \emph{strongly decisive} learning ($\SDec$, \cite{Koe:c:14:stacs}) the learner may not even (semantically) return to \emph{syntactically} abandoned conjectures. Finally, a number of monotonicity requirements are studied (\cite{Jan:j:91,Wie:c:91,Lan-Zeu:c:93}): in \emph{strongly monotone} learning ($\SMon$) the conjectured sets may only grow; in \emph{monotone} learning ($\Mon$) only incorrect data may be removed; and in \emph{weakly monotone} learning ($\WMon$) the conjectured set may only grow while it is consistent.

The main question is now whether and how these different restrictions reduce learning power. For example, non-U-shaped learning is known not to restrict the learning power \cite{Bal-Cas-Mer-Ste-Wie:j:08}, and the same for strongly non-U-shaped learning \cite{Cas-Moe:j:11:optLan}; on the other hand, decisive learning \emph{is} restrictive \cite{Bal-Cas-Mer-Ste-Wie:j:08}. The relations of the different monotone learning restriction were given in \cite{Lan-Zeu:c:93}. Conservativeness is long known to restrict learning power \cite{Ang:j:80:lang-pos-data}, but also known to be equivalent to weakly monotone learning \cite{Kin-Ste:j:95:mon,Jai-Sha:j:98}. 

Cautious learning was shown to be a restriction but not when added to conservativeness in \cite{Osh-Sto-Wei:j:82:strategies,Osh-Sto-Wei:b:86:stl}, similarly the relationship between decisive and conservative learning was given. In Exercise 4.5.4B of \cite{Osh-Sto-Wei:b:86:stl} it is claimed (without proof) that cautious learners cannot be made conservative; we claim the opposite in Theorem~\ref{thm:CautVarConv}.

This list of previously known results leaves a number of relations between the learning criteria open, even when adding trivial inclusion results (we call an inclusion trivial iff it follows straight from the definition of the restriction without considering the learning model, for example strongly decisive learning is included in decisive learning; formally, trivial inclusion is inclusion on the level of learning restrictions as predicates, see Section~\ref{sec:LearningCriteria}).
With this paper we now give the complete picture of these learning restrictions. The result is shown as a map in Figure~\ref{fig:GoldRelations}. A solid black line indicates a trivial inclusion (the lower criterion is included in the higher); a dashed black line indicates inclusion (which is not trivial). A gray box around criteria indicates equality of 
(learning of) these criteria.

\begin{figure}[ht]
\begin{center}
\includegraphics[viewport = 4.5cm 14cm 19cm 24cm, clip]{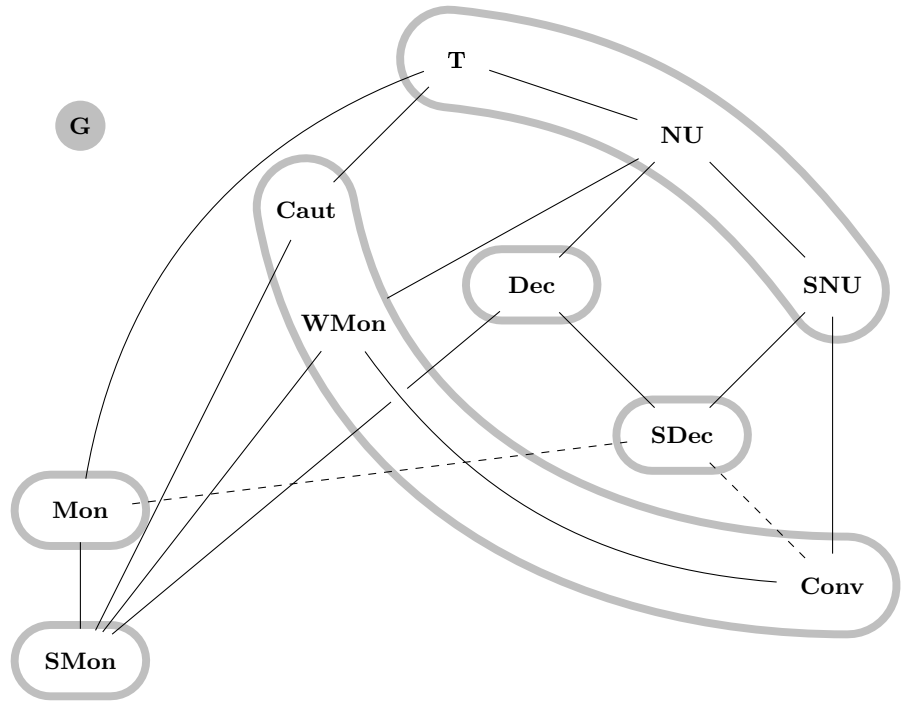}
\end{center}
\caption{Relation of criteria.}\label{fig:GoldRelations}
\end{figure}

A different way of depicting the same results is given in Figure~\ref{fig:partialorderGold} (where solid lines indicate any kind of inclusion). Results involving monotone learning can be found in Section~\ref{sec:Monotone}, all others in Section~\ref{sec:Caution}.

\begin{figure}[h]
\begin{center}
\begin{tikzpicture}[above,sloped,shorten <=2mm,, shorten >=2mm]

\begin{scope}[every node/.style={rectangle, rounded corners,minimum size=4mm,draw=lightgray,line width=2pt}]

\node[shape=circle,fill=lightgray] at (-5,-1) {$\Gold$};

\node (nothing) at (0,0)  {\textbf{T} \hspace{4mm} $\NU$ \hspace{4mm} $\SNU$};
\node (dec)  at (0,-1.5)   		{$\Dec$};
\node (sdec) at (0,-3)   {$\SDec$};
\node (mon) at (2.5,-4.5)	 {$\Mon$}; 
\node (wmon) at (-2.5,-4.5)    {$\Caut$ \hspace{4mm} $\WMon$ \hspace{4mm} $\Conv$};
\node (smon) at (0,-6) 	  {$\SMon$};

\draw (nothing) -- (dec);
\draw (dec) -- (sdec);
\draw (sdec) -- (wmon);
\draw (sdec) -- (mon); 
\draw (mon) -- (smon);
\draw (wmon) -- (smon);

\end{scope}

\end{tikzpicture}
\end{center}
\caption{Partial order of delayable learning restrictions in Gold-style learning.}\label{fig:partialorderGold}
\end{figure}
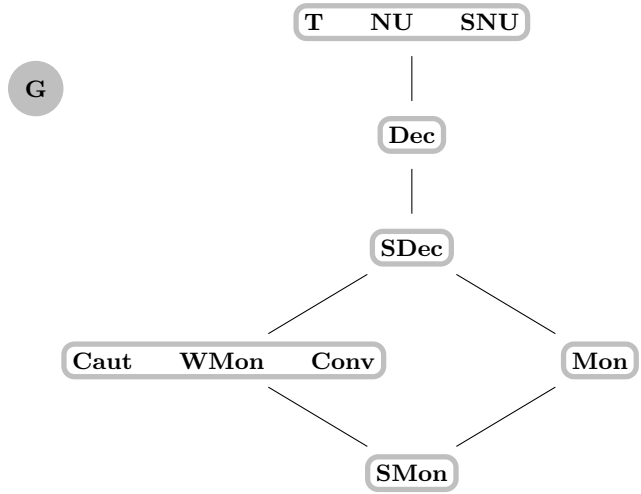

For the important restriction of conservative learning we give the characterization of being equivalent to cautious learning. Furthermore, we show that even two weak versions of cautiousness are equivalent to conservative learning. Recall that cautiousness forbids to return to a strict subset of a previously conjectured set. If we now weaken this restriction to forbid to return to \emph{finite} subsets of a previously conjectured set we get a restriction still equivalent to conservative learning. If we forbid to go down to a correct conjecture, effectively forbidding to ever conjecture a superset of the target language, we also obtain a restriction equivalent to conservative learning. On the other hand, if we weaken it so as to only forbid going to \emph{infinite} subsets of previously conjectured sets, we obtain a restriction equivalent to no restriction. These results can be found in Section~\ref{sec:Caution}.

In \emph{set-driven} learning \cite{Wex-Cul:b:80} the learner does not get the full information about what data has been presented in what order and multiplicity; instead, the learner only gets the set of data presented so far. For this learning model it is known that, surprisingly, conservative learning is no restriction \cite{Kin-Ste:j:95:mon}! We complete the picture for set driven learning by showing that set-driven learners can always be assumed conservative, strongly decisive and cautious, and by showing that the hierarchy of monotone and strongly monotone learning also holds for set-driven learning. The situation is depicted in Figure~\ref{fig:hierarchySetdriven}. These results can be found in Section~\ref{sec:SetDriven}.

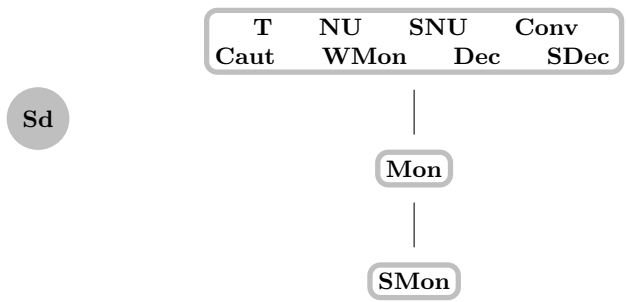
\begin{figure}[h]
\begin{center}
\begin{tikzpicture}[above,sloped,shorten <=2mm,, shorten >=2mm]

\begin{scope}[every node/.style={rectangle, rounded corners,minimum size=4mm,draw=lightgray,line width=2pt}]

\node[shape=circle,fill=lightgray] at (-5,-1) {$\Sd$};

\node (nothing) at (0,0)  {\parbox{5.3cm}{\hspace{4mm} \textbf{T} \hspace{4mm} $\NU$ \hspace{4mm} $\SNU$ \hspace{4mm} $\Conv$ \\ $\Caut$ \hspace{4mm} $\WMon$ \hspace{4mm} $\Dec$ \hspace{4mm} $\SDec$}};
\node (mon) at (0,-1.5)		  {$\Mon$}; 
\node (smon) at (0,-3)	  {$\SMon$};

\draw (nothing) -- (mon);
\draw (mon) -- (smon);

\end{scope}

\end{tikzpicture}
\end{center}
\caption{Hierarchy of delayable learning restrictions in set-driven learning}\label{fig:hierarchySetdriven}
\end{figure}

\subsection{Techniques}

A major emphasis of this paper is on the techniques used to get our results. These techniques include specific techniques for specific problems, as well as general theorems which are applicable in many different settings. The general techniques are given in Section~\ref{sec:techniques}, one main general result is as follows. It is well-known that any $\Text\Gold\Ex$-learner $h$ learning a language $L$ has a \emph{locking sequence}, a sequence $\sigma$ of data from $L$ such that, for any further data from $L$, the conjecture does not change and is correct. However, there might be texts such that no initial sequence of the text is a locking sequence. We call a learner such that any text for a target language contains a locking sequence \emph{strongly locking}, a property which is very handy to have in many proofs. Fulk \cite{Ful:j:90:prudence} showed that, without loss of generality, a $\Text\Gold\Ex$-learner can be assumed strongly locking, as well as having many other useful properties (we call this the \emph{Fulk normal form}, see Definition~\ref{defn:FulkNormalForm}). For many learning criteria considered in this paper it might be too much to hope for that all of them allow for learning by a learner in Fulk normal form. However, we show in Corollary~\ref{cor:SinkLocking} that we can always assume our learners to be strongly locking, total, and what we call \emph{syntactically decisive}, never \emph{syntactically} returning to syntactically abandoned hypotheses.

The main technique we use to show that something is decisively learnable, for example in Theorem~\ref{thm:NatnumSDec}, is what we call \emph{poisoning} of conjectures. In the proof of Theorem~\ref{thm:NatnumSDec} we show that a class of languages is decisively learnable by simulating a given monotone learner $h$, but changing conjectures as follows. Given a conjecture $e$ made by $h$, if there is no mind change in the future with data from conjecture $e$, the new conjecture is equivalent to $e$; otherwise it is suitably changed, \emph{poisoned}, to make sure that the resulting learner is decisive. This technique was also used in \cite{Cas-Koe:c:10:colt} to show strongly non-U-shaped learnability.

Finally, for showing classes of languages to be not (strongly) decisively learnable, we adapt a technique known in computability theory as a ``priority argument'' (note, though, that we do not deal with oracle computations). We use this technique to reprove that decisiveness is a restriction to $\Text\Gold\Ex$-learning (as shown in \cite{Bal-Cas-Mer-Ste-Wie:j:08}), and then use a variation of the proof to show that strongly decisive learning is a restriction to decisive learning.

\section{Mathematical Preliminaries}
\setboolean{useproof}{true}

\setboolean{useproof}{true}

\label{sec:MathPrelim}

Unintroduced notation follows \cite{Rog:b:87}, a textbook on computability theory. 

$\natnum$ denotes the set of natural numbers, $\{0,1,2,\ldots\}$. The symbols $\subseteq$, $\subset$, $\supseteq$, $\supset$ respectively denote the subset, proper subset, superset and proper superset relation between sets; $\setminus$ denotes set difference. 
$\emptyset$ and $\emptysequence$ denote the empty set and the empty sequence, respectively. The quantifier $\forall^\infty x$ means ``for all but finitely many $x$''. With $\dom$ and $\range$ we denote, respectively, domain and range of a given function. 

Whenever we consider tuples of natural numbers as input to a function, it is understood that the general coding function $\langle \cdot, \cdot \rangle$ is used to code the tuples into a single natural number. We similarly fix a coding for finite sets and sequences, so that we can use those as input as well. For finite sequences, we suppose that for any $\sigma \subseteq \tau$ we have that the code number of $\sigma$ is at most the code number of $\tau$. We let $\seq$ denote the set of all (finite) sequences, and $\seq_{\leq t}$ the (finite) set of all sequences of length at most $t$ using only elements $\leq t$.

If a function $f$ is not defined for some argument $x$, then we denote this fact by $f(x)\diverges$, and we say that $f$ on $x$ \emph{diverges}; the opposite is denoted by $f(x)\converges$, and we say that $f$ on $x$ \emph{converges}. If $f$ on $x$ converges to $p$, then we denote this fact by $f(x)\converges = p$. We let $\FrakP$ denote the set of all partial functions $\natnum \rightarrow \natnum$ and $\FrakR$ the set of all total such functions.

$\CalP$ and $\CalR$ denote, respectively, the set of all partial computable and the set of all total computable functions (mapping $\natnum \rightarrow \natnum$).

We let $\varphi$ be any fixed acceptable programming system for $\CalP$ (an acceptable programming system could, for example, be based on a natural programming language such as C or Java, or on Turing machines). Further, we let $\varphi_p$ denote the partial computable function computed by the $\varphi$-program with code number $p$.
A set $L \subseteq \natnum$ is \emph{computably enumerable (\ce)} iff it is the domain of a computable function. Let $\CalE$ denote the set of all \ce\ sets. We let $W$ be the mapping such that $\forall e: W(e) = \dom(\varphi_e)$. For each $e$, we write $W_e$ instead of $W(e)$. $W$ is, then, a mapping from $\natnum$ \emph{onto} $\CalE$. We say that $e$ is an index, or program, (in $W$) for $W_e$.

We let $\Phi$ be a Blum complexity measure associated with $\varphi$ (for example, for each $e$ and $x$, $\Phi_e(x)$ could denote the number of steps that program $e$ takes on input $x$ before terminating). For all $e$ and $t$ we let $W_e^t = \set{x \leq t}{\Phi_e(x) \leq t}$ (note that a complete description for the finite set $W_e^t$ is computable from $e$ and~$t$).
The symbol $\#$ is pronounced \emph{pause} and is used to symbolize ``no new input data'' in a text. For each (possibly infinite) sequence $q$ with its range contained in $\natnum \cup \{\#\}$, let $\content(q) = (\range(q) \setminus \{\#\}$). By using an appropriate coding, we assume that $?$ and $\#$ can be handled by computable functions.
For any function $T$ and all $i$, we use $T[i]$ to denote the sequence $T(0)$, \ldots, $T(i-1)$ (the empty sequence if $i=0$ and undefined, if any of these values is undefined).

\iffalse
We will use Case's \emph{Operator Recursion Theorem} (\ORT), providing \emph{infinitary} self-and-other program reference \cite{Cas:j:74,Cas:j:94:self,Jai-Osh-Roy-Sha:b:99:stl2}. \ORT\ itself states that, for all operators $\Theta$ there are $f$ with $\forall z: \Theta(\varphi_z) = \varphi_{f(z)}$ and $e \in \CalR$,
\begin{equation*}\label{eq:ORT}
\forall a,b: \varphi_{e(a)}(b) = \Theta(e)(a,b).
\end{equation*}
\fi

\subsection{Learning Criteria}
\setboolean{useproof}{true}

\label{sec:LearningCriteria}

In this section we formally introduce our setting of learning in the limit and associated learning criteria. We follow \cite{Koe:th:09} in its ``building-blocks'' approach for defining learning criteria.

A \emph{learner} is a partial computable function $h \in \CalP$. 
A \emph{language} is a \ce\ set $L \subseteq \natnum$. Any total function $T: \natnum \rightarrow \natnum \cup \{\#\}$ is called a \emph{text}. For any given language $L$, a \emph{text for $L$} is a text $T$ such that $\content(T) = L$. Initial parts of this kind of text is what learners usually get as information. 

An \emph{interaction operator} is an operator $\beta$ taking as arguments a function $h$ (the learner) and a text $T$, and that outputs a function $p$. We call $p$ the \emph{learning sequence} (or \emph{sequence of hypotheses}) of $h$ given $T$. Intuitively, $\beta$ defines how a learner can interact with a given text to produce a sequence of conjectures.

We define the interaction operators $\Gold$, $\Psd$ (partially set-driven learning, \cite{Sch:th:84}) and $\Sd$ (set-driven learning, \cite{Wex-Cul:b:80}) as follows. For all learners $h$, texts $T$ and all $i$,
\begin{eqnarray*}
\Gold(h,T)(i) & = & h(T[i]);\\
\Psd(h,T)(i)  & = & h(\content(T[i]),i);\\
\Sd(h,T)(i)   & = & h(\content(T[i])).
\end{eqnarray*}
Thus, in set-driven learning, the learner has access to the set of all previous data, but not to the sequence as in $\Gold$-learning. In partially set-driven learning, the learner has the set of data and the current iteration number.

Successful learning requires the learner to observe certain restrictions, for example convergence to a correct index. These restrictions are formalized in our next definition.

A \emph{learning restriction} is a predicate $\delta$ on a learning sequence and a text. 
We give the important example of explanatory learning ($\Ex$, \cite{Gol:j:67}) defined such that, for all sequences of hypotheses $p$ and all texts $T$,
\begin{align*}
\Ex(p,T) \Leftrightarrow & \; p \mbox{ total }\wedge [\exists n_0 \forall n \geq n_0: p(n) = p(n_0) \wedge W_{p(n_0)} = \content(T)].
\end{align*}
Furthermore, we formally define the restrictions discussed in Section~\ref{sec:Introduction} in Figure~\ref{fig:DefinitionsOfLearningRestrictions} (where we implicitly require the learning sequence $p$ to be total, as in $\Ex$-learning; note that this is a technicality without major importance).

\begin{figure}[h]
\begin{align*}
\Conv(p,T) \Leftrightarrow & \; [\forall i: \content(T[i+1]) 
  \subseteq W_{p(i)} \Rightarrow p(i) = p(i+1)];\\
\Caut(p,T) \Leftrightarrow & \; [\forall i,j: W_{p(i)} 
  \subset W_{p(j)} \Rightarrow i < j];\\
\NU(p,T) \Leftrightarrow & \; [\forall i,j,k: i \leq j \leq k \; \wedge \; 
  W_{p(i)} = W_{p(k)} = \content(T) \Rightarrow W_{p(j)} = W_{p(i)}];\\
\Dec(p,T) \Leftrightarrow & \; [\forall i,j,k: i \leq j \leq k \; \wedge \; 
  W_{p(i)} = W_{p(k)} \Rightarrow W_{p(j)} = W_{p(i)}];\\
\SNU(p,T) \Leftrightarrow & \; [\forall i,j,k: i \leq j \leq k \; \wedge \; 
  W_{p(i)} = W_{p(k)} = \content(T) \Rightarrow p(j) = p(i)];\\
\SDec(p,T) \Leftrightarrow & \; [\forall i,j,k: i \leq j \leq k \; \wedge \; 
  W_{p(i)} = W_{p(k)} \Rightarrow p(j) = p(i)];\\
\SMon(p,T) \Leftrightarrow & \; [\forall i,j: i< j \Rightarrow W_{p(i)} 
   \subseteq W_{p(j)}];\\
\Mon(p,T) \Leftrightarrow & \; [\forall i,j: i< j \Rightarrow W_{p(i)}
  \cap\content(T) \subseteq W_{p(j)}\cap\content(T)];\\
\WMon(p,T) \Leftrightarrow & \; [\forall i,j: i < j \wedge \content(T[j]) 
 \subseteq W_{p(i)} \Rightarrow W_{p(i)} \subseteq W_{p(j)}].
\end{align*}\vspace{-7mm}
\caption{Definitions of learning restrictions.}
\label{fig:DefinitionsOfLearningRestrictions}\end{figure}
A variant on decisiveness is \emph{syntactic decisiveness}, $\SynDec$, a technically useful property defined as follows.
$$
\SynDec(p,T) \Leftrightarrow  [\forall i,j,k: i \leq j \leq k \; \wedge \; p(i) = p(k) \Rightarrow p(j) = p(i)].
$$
We combine any two sequence acceptance criteria $\delta$ and $\delta'$ by intersecting them; we denote this by juxtaposition (for example, all the restrictions given in Figure~\ref{fig:DefinitionsOfLearningRestrictions} are meant to be always used together with $\Ex$). With $\mathbf{T}$ we denote the always true sequence acceptance criterion (no restriction on learning).

A \emph{learning criterion} is a tuple $(\CalC,\beta,\delta)$, where $\CalC$ is a set of learners (the admissible learners), $\beta$ is an interaction operator and $\delta$ is a learning restriction; we usually write $\CalC\Text\beta\delta$ to denote the learning criterion, omitting $\CalC$ in case of $\CalC = \CalP$. We say that a learner $h \in \CalC$ \emph{$\CalC\Text\beta\delta$-learns} a language $L$ iff, for all texts $T$ for $L$, $\delta(\beta(h,T),T)$. The set of languages $\CalC\Text\beta\delta$-learned by $h \in \CalC$ is denoted by $\CalC\Text\beta\delta(h)$.
We write $[\CalC\Text\beta\delta]$ to denote the set of all $\CalC\Text\beta\delta$-learnable classes (learnable by some learner in $\CalC$). 

\section{Delayable Learning Restrictions}
\setboolean{useproof}{true}

\label{sec:techniques}

In this section we present technically useful results which show that learners can always be assumed to be in some normal form. We will later always assume our learners to be in the normal form established by Corollary~\ref{cor:SinkLocking}, the main result of this section.

\setboolean{useproof}{true}

We start with the definition of \emph{delayable}. Intuitively, a learning criterion $\delta$ is delayable iff the output of a hypothesis can be arbitrarily (but not indefinitely) delayed.

\begin{defn}\label{defn:Delayable}
Let $\vec{R}$ be the set of all non-decreasing $r: \natnum \rightarrow \natnum$ with infinite limit inferior, i.e.\ for all $m$ we have $\forall^\infty n: r(n) \geq m$.

A learning restriction $\delta$ is \emph{delayable} iff, for all texts $T$ and $T'$ with $\content(T) = \content(T')$, all $p$ and all $r \in \vec{R}$,
if $(p,T) \in \delta$ and $\forall n: \content(T[r(n)]) \subseteq \content(T'[n])$, then  $(p \circ r,T') \in \delta$. Intuitively, as long as the learner has at least as much data as was used for a given conjecture, then the conjecture is permissible. Note that this condition holds for $T = T'$ if $\forall n: r(n) \leq n$.
\end{defn}

Note that the intersection of two delayable learning criteria is again delayable and that \emph{all} learning restrictions considered in this paper are delayable.

\setboolean{useproof}{true}

As the name suggests, we can apply \emph{delaying tricks} (tricks which delay updates of the conjecture) in order to achieve fast computation times in each iteration (but of course in the limit we still spend an infinite amount of time). This gives us equally powerful but total learners, as shown in the next theorem. While it is well-known that, for many learning criteria, the learner can be assumed total, this theorem explicitly formalizes conditions under which totality can be assumed (note that there are also natural learning criteria where totality cannot be assumed, such as consistent learning \cite{Jai-Osh-Roy-Sha:b:99:stl2}).

\begin{thm}\label{thm:Total}
For any delayable learning restriction $\delta$, we have [$\Text\Gold\delta$] = [$\CalR\Text\Gold\delta$].
\end{thm}
\begin{proof}
Let $h$ be a $\Text\Gold\delta$-learner and $e$ such that $\varphi_e = h$. We define a function $M$ such that, for all $\sigma$,
$$M(\sigma) = \{\sigma' \subseteq \sigma\ |\ \Phi_e(\sigma') \leq |\sigma|\} \cup \{\emptysequence\}.$$ We let $h'$ be the learner such that, for all $\sigma$,
$$h'(\sigma) = h(\max(M(\sigma)).$$
As $h$ is required to have only total learning sequences, we have that $h(\emptysequence)\converges$; thus, $h'$ is total computable using that $M$ is total computable.
Let $\CalL = \Text\Gold\delta(h)$, $L \in \CalL$ and let $T$ be a text for $L$. 
Let $r(n) = |\max(M(T[n]))|$. Then we have, for all $n$, $h'(T[n]) = h(T[r(n)])$. Thus, if we show that $r \in \vec{R}$ we get that $h'$ $\Text\Gold\delta$-learns $L$ from $T$ using $\delta$ delayable. From the definition of $M$ we get that $r$ is non-decreasing and, for all $n$, $r(n) \leq n$. For any given $m$ there are $n,n'$ with $n' \geq n \geq m$ such that $\Phi_e(T[n]) \leq n'$. Thus, we have $r(n') \geq m$ and, as $r$ is non-decreasing, we get $\forall^\infty n : r(n) \geq m$ as desired. 
\end{proof}

\setboolean{useproof}{true}

Next we define another useful property, which can always be assumed for delayable learning restrictions.

\begin{defn}\label{defn:StronglyLocking}
A \emph{locking sequence for a learner $h$ on a language $L$} is any finite sequence $\sigma$ of elements from $L$ such that $h(\sigma)$ is a correct hypothesis for $L$ and, for sequences $\tau$ with elements from $L$, $h(\sigma \diamond \tau) = h(\sigma)$\cite{Blu-Blu:j:75}. It is well known that every learner $h$ learning a language $L$ has a locking sequence on $L$.
We say that a learning criterion $I$ \emph{allows for strongly locking learning} iff, for each $I$-learnable class of languages $\CalL$ there is a learner $h$ such that $h$ $I$-learns $\CalL$ and, for each $L \in \CalL$ and any text $T$ for $L$, there is an $n$ such that $T[n]$ is a locking sequence of $h$ on $L$  (we call such a learner $h$ \emph{strongly locking}).
\end{defn}

\setboolean{useproof}{true}

With this definition we can give the following theorem.

\begin{thm}\label{thm:DelayStronglyLocking}
Let $\delta$ be a delayable learning criterion. Then $\CalR\Text\Gold\delta\Ex$ allows for strongly locking learning.
\end{thm}
\begin{proof}
Let $\CalL$ and $h \in \CalR$ be such that $h$ $\CalR\Text\Gold\delta\Ex$-learns $\CalL$. We define a set $M(\rho, \sigma)$, for all $\rho$ and $\sigma$ such that
$$M(\rho,\sigma)  =  \set{\tau}{|\tau| \leq |\sigma| \wedge \content(\tau) \subseteq \content(\sigma) \wedge h(\rho \diamond \tau) \neq h(\rho)}.$$
Thus, $M$ contains sequences with elements from $\content(\sigma)$ such that $h$ makes a mind change on $\sigma$ extended with such a sequence. Additionally, we define a function $f$ recursively such that, for all $\sigma, x$ and $T$,
\begin{eqnarray*}
f(\emptyset) & = & \emptyset; \\
f(\sigma \diamond x) & = &
\begin{cases} 
f(\sigma), 															&\mbox{if }M(f(\sigma),\sigma \diamond x) = \emptyset;\\
f(\sigma) \diamond \min(M(f(\sigma),\sigma \diamond x)) \diamond \sigma, 		&\mbox{otherwise;}
\end{cases}\\
f(T) & = & \lim\limits_{n \rightarrow \infty}{f(T[n])}.
\end{eqnarray*}
Intuitively, $f$ searches for longer and longer sequences which are \emph{not} locking sequences. We let $h'$ be the learner such that, for all $\sigma$,
$$h'(\sigma)  =  h(f(\sigma)).$$
Note that $f$ is total (as $h$ is total), and thus $h'$ is total.

Let $L \in \CalL$ and $T$ be a text for $L$. We will show now that $f(T)$ converges to a finite sequence.
\begin{claim}
We have that $f(T)$ is finite.
\end{claim}
\begin{claimProof}
By way of contradiction, suppose that $f(T)$ is infinite, and let $T' = f(T)$. As $f(T)$ is infinite we get, for every $n$, an $m > n$ such that $f(T[m]) \neq f(T[n])$. Then we have 
$$\content(T[n]) \subseteq \content(f(T[m])).$$ 
As this holds for every $n$, we get $\content(T) \subseteq \content(f(T))$. From the construction of $f$ we know that $\content(f(T)) \subseteq \content(T)$. Thus, $f(T)$ is a text for $L$. From the construction of $M$ we get that $h$ does not $\Text\Gold\Ex$-learns $L$ from $T'$ as $h$ changes infinitely often its mind, a contradiction.
\end{claimProof}
Next, we will show that $h'$ converges on $T$ and $h'$ is strongly locking. As $f(T)$ is finite, there is $n_0$ such that, for all $n \geq n_0$,
\begin{align*}
f(T[n]) = f(T[n_0]).
\end{align*}

As $f(T)$ converges to $f(T[n_0])$, we get from the construction of $M$ that $f(T[n_0])$ is a locking sequence of $h$ on $L$. Therefore we get that, for all $\tau \in \seq(L)$, 
$$f(T[n_0]) = f(T[n_0] \diamond \tau)$$ and therefore
$$h'(T[n_0]) = h'(T[n_0] \diamond \tau).$$
Thus, $h'$ is strongly locking and converges on $T$.

To show that $h'$ fulfills the $\delta$-restriction, we let $T' = f(T[n_0]) \diamond T$ be a text for $L$ starting with $f(T[n_0])$. Let $r$ be such that
$$r(n) = \begin{cases}
	|f(T[n])|, & \text{if } n \leq n_0; \\
	r(n_0) + n - n_0, & \text{otherwise.}
\end{cases}$$We now show
$$h(T'[r(n)]) = h'(T[n]).$$

\textit{Case 1:} $n \leq n_0$. Then we get 
\begin{align*}
h(T'[r(n)]) &= h(T'[|f(T[n])|]) \\
	&= h(f(T[n])) &\text{as $T' = f(T[n_0]) \diamond T$} \\
	&= h'(T[n]).
\end{align*}

\textit{Case 2:} $n > n_0$. Then we get
\begin{align*}
h(T'[r(n)]) &= h(T'[r(n_0) + n - n_0])\\
	& = h(T'[|f(T[n_0])| + n - n_0]) \\ 
	&= h(f(T[n_0])\diamond T[n-n_0]) &\text{as $T' = f(T[n_0]) \diamond T$}\\
	&= h(f(T[n_0])) &\text{\hspace{-10mm}as $f(T[n_0])$ is a locking sequence of $h$} \\
	&= h'(T[n]).
\end{align*}
 Thus, all that remains to be shown is that $r \in \vec{R}$. Obviously, $r$ is non-decreasing. Especially, we have that $r$ is strongly monotone increasing for all $n > n_0$. Thus we have, for all $m$, $\forall^\infty n : r(n) \geq m$.
Finally we show that $\content(T'[r(n)]) \subseteq \content(T[n])$. From the construction of $f$ we have, for all $n \leq n_0$, $\content(T'[|f(T[n])|]) \subseteq \content(T[n])$. From the construction of $r$ and $T'$ we get that, for all $n > n_0$, $T'(r(n)) = T(n)$. Thus we get, for all $n$, $\content(T'[r(n)]) \subseteq \content(T[n])$.

\end{proof}

\setboolean{useproof}{true}

Next we define semantic and pseudo-semantic restrictions introduced in~\cite{Koe:c:14:stacs}. Intuitively, semantic restrictions allow for replacing hypotheses by equivalent ones; pseudo-sematic restrictions allow the same, as long as no new mind changes are introduced.

\begin{defn}\label{defn:SemanticRestriction}
For all total functions $p \in \FrakP$, we let
\begin{eqnarray*}
\SemRestr(p) & = & \set{ p' \in \FrakP}{\forall i: W_{p(i)} = W_{p'(i)}};\\
\McRestr(p) & = & \set{ p' \in \FrakP}{\forall i: p'(i) \neq p'(i+1) \Rightarrow p(i) \neq p(i+1)}.
\end{eqnarray*}

A sequence acceptance criterion $\delta$ is said to be a \emph{semantic restriction} iff, for all $(p,q) \in \delta$ and $p' \in \SemRestr(p)$, $(p',q) \in \delta$.

A sequence acceptance criterion $\delta$ is said to be a \emph{pseudo-semantic restriction} iff, for all $(p,q) \in \delta$ and $p' \in \SemRestr(p) \cap \McRestr(p)$, $(p',q) \in \delta$.

\end{defn}

We note that the intersection of two (pseudo-) semantic learning restrictions is again (pseudo-) semantic. All learning restrictions considered in this paper are pseudo-semantic, and all except $\Conv$, $\SNU$, $\SDec$ and $\Ex$ are semantic.

\setboolean{useproof}{true}

The next lemma shows that, for every pseudo-semantic learning restriction, learning can be done syntactically decisively.

\begin{lem}\label{lem:SynDec}
Let $\delta$ be a pseudo-semantic learning criterion. Then we have
$$[\CalR\Text\Gold\delta] = [\CalR\Text\Gold\SynDec\delta].$$
\end{lem}
\begin{proof}
Let a $\Text\Gold\delta$-learner $h \in \CalR$ be given. We define a learner $h' \in \CalR$ such that, for all $\sigma$, 
$$h'(\sigma) = \begin{cases}
	\pad(h(\sigma),\sigma), & \text{if } \sigma = \emptyset \text{ or } h(\sigma) \neq h(\sigma^-); \\
	h'(\sigma^-), & \text{otherwise.}
\end{cases}$$
The correctness of this construction is straightforward to check.
\end{proof}

\setboolean{useproof}{true}

As $\SynDec$ is a delayable learning criterion, we get the following corollary by taking Theorems~\ref{thm:Total} and~\ref{thm:DelayStronglyLocking} and Lemma~\ref{lem:SynDec} together. We will always assume our learners to be in this normal form in this paper.

\begin{cor}\label{cor:SinkLocking}
Let $\delta$ be pseudo-semantic and delayable. Then $\Text\Gold\delta\Ex$ allows for strongly locking learning  by a syntactically decisive total learner.
\end{cor}

\setboolean{useproof}{true}

Fulk showed that any $\Text\Gold\Ex$-learner can be (effectively) turned into an equivalent learner with many useful properties, including strongly locking learning \cite{Ful:j:90:prudence}. One of the properties is called \emph{order-independence}, meaning that on any two texts for a target language the learner converges to the same hypothesis. Another property is called \emph{rearrangement-independence}, where a learner $h$ is rearrangement-independent if there is a function $f$ such that, for all sequences $\sigma$, $h(\sigma) = f(\content(\sigma),|\sigma|)$ (intuitively, rearrangement independence is equivalent to the existence of a partially set-driven learner for the same language). We define the collection of all the properties which Fulk showed a learner can have to be the \emph{Fulk normal form} as follows.

\begin{defn}\label{defn:FulkNormalForm}
We say a $\Text\Gold\Ex$-learner $h$ is in \emph{Fulk normal form} if $(1) - (5)$ hold.
\begin{enumerate}
\item $h$ is order-independent.
\item $h$ is rearrangement-independent.
\item If $h$ $\Text\Gold\Ex$-learns a language $L$ from some text, then $h$ $\Text\Gold\Ex$-learns~$L$.
\item If there is a locking sequence of $h$ for some $L$, then $h$ $\Text\Gold\Ex$-learns $L$.
\item For all $\CalL \in \Text\Gold\Ex(h)$,  $h$ is strongly locking on $\CalL$.
\end{enumerate}
\end{defn}

The following theorem is somewhat weaker than what Fulk states himself.
\begin{thm}[{\cite[Theorem 13]{Ful:j:90:prudence}}]\label{thm:FulkNormalForm}
Every $\Text\Gold\Ex$-learnable set of languages has a $\Text\Gold\Ex$-learner in Fulk normal form.
\end{thm}

\section{Full-Information Learning}
\setboolean{useproof}{true}

\label{sec:Caution}

In this section we consider various versions of cautious learning and show that all of our variants are either no restriction to learning, or equivalent to conservative learning as is shown in Figure~\ref{fig:GoldCautiousV}.

Additionally, we will show that every cautious $\Text\Gold\Ex$-learnable language is conservative $\Text\Gold\Ex$-learnable which implies that $[\Text\Gold\Conv\Ex]$, $[\Text\Gold\WMon\Ex]$ and $[\Text\Gold\Caut\Ex]$ are equivalent. Last, we will separate these three learning criteria from strongly decisive $\Text\Gold\Ex$-learning and show that $[\Text\Gold\SDec\Ex]$ is a proper superset.

\setboolean{useproof}{true}

\begin{thm}\label{thm:ConvInSDec}
We have that any conservative learner can be assumed cautious and strongly decisive, i.e.
$$[\textbf{TxtGConvEx}] = [\textbf{TxtGConvSDecCautEx}].$$
\end{thm}
\begin{proof}
Let $h \in \CalR$ and $\mathcal{L}$ be such that $h$ \textbf{TxtGConvEx}-learns $\mathcal{L}$. We define, for all $\sigma$, a set $M(\sigma)$ as follows
$$M(\sigma) = \{\tau\ |\ \tau \subseteq \sigma\ \land\ \forall x \in \content(\tau) : \Phi_{h(\tau)}(x) \leq \left|\sigma\right| \}.$$
We let 
$$\forall \sigma : h'(\sigma) = h(\max(M(\sigma))).$$
Let $T$ be a text for a language $L \in \mathcal{L}$. We first show that $h'$ \textbf{TxtGEx}-learns $L$ from the text $T$. As $h$ \textbf{TxtGConvEx}-learns $L$, there are $n$ and $e$ such that $\forall n' \geq n : h(T[n]) = h(T[n']) = e$ and $W_e = L$. Thus, there is $m \geq n$ such that $\forall x \in \content(T[n]) : \Phi_{h(T[n])}(x) \leq m$ and therefore $\forall m' \geq m : h'(T[m])=h'(T[m'])=e$. \par 
Next we show that $h'$ is strongly decisive and conservative; for that we show that, with every mind change, there is a new element of the target included in the conjecture which is currently not included but is included in all future conjectures; it is easy to see that this property implies both caution and strong decisiveness. 
Let $i$ and $i'$ be such that $\max(M(T[i'])) = T[i]$. This implies that
 $$\content(T[i]) \subseteq W_{h'(T[i'])}.$$
Let $j' > i'$ such that $h'(T[i']) \neq h'(T[j'])$. Then there is $j > i$ such that $\max(M(T[j'])) = T[j]$ and therefore 
$$\content(T[j]) \subseteq W_{h'(T[j'])}.$$
Note that in the following diagram $j$ could also be between $i$ and $i'$.
\begin{center}
\begin{tikzpicture}[scale = 0.9]

\node (left) at (-1,0)[]{};
\node (right) at (12,0)[]{};
\node (labelil) at (4,-1)[]{$h'(T[i']) = h(T[i])$};
\node (labeljl) at (10,-1)[]{$h'(T[j']) = h(T[j])$};
\node at (4,-1.5) {$\content(T[i]) \subseteq W_{h(T[i])}$};
\node at (10,-1.5) {$\content(T[j]) \subseteq W_{h(T[j])}$};

\draw[] (left) -- (right);
\draw[] (1,1pt) -- (1,-1pt) node[anchor=south]{$i$} node[anchor=north]{mind change $h$};
\draw[] (4,1pt) -- (4,-1pt) node[anchor=south]{$i'$} node[anchor=north]{mind change $h'$};
\draw[] (7,1pt) -- (7,-1pt) node[anchor=south]{$j$} node[anchor=north]{mind change $h$};
\draw[] (10,1pt) -- (10,-1pt) node[anchor=south]{$j'$} node[anchor=north]{mind change $h'$};
\draw[thick,decorate,decoration={brace,amplitude=12pt}] (4,0.5) -- (10,0.5) node[midway, above,yshift=12pt,]{no mind change $h'$};

\end{tikzpicture}
\end{center}
As $h$ is conservative and $\content(T[i]) \subseteq W_{h(T[i])}$, there exists $\ell$ such that $i < \ell < j$ and $T(\ell) \notin W_{h(T[i])}$. Then we have $\forall n \geq j' : T(\ell) \in W_{h'(T[n])}$ as $T(\ell) \in W_{h'(T[j'])}$. 

Obviously $h'$ is conservative as it only outputs (delayed) hypotheses of $h$ (and maybe skip some) and $h$ is conservative.
\end{proof}

\setboolean{useproof}{true}

In the following we consider three new learning restrictions. The learning restriction $\Caut_\mathbf{Fin}$ means that the learner never returns a hypothesis for a finite set that is a proper subset of a previous hypothesis. $\Caut_\infty$ is the same restriction for infinite hypotheses. With $\Caut_\mathbf{Tar}$ the learner is not allowed to ever output a hypothesis that is a proper superset of the target language that is learned. 

\begin{defn}\label{defn:CautVariations}
\begin{align*}
\Caut_{\mathbf{Fin}}(p,T) &\Leftrightarrow [\forall i < j: W_{p(j)} \subset W_{p(i)} \Rightarrow W_{p(j)} \text{ is infinite}]\\
\Caut_{\infty}(p,T) &\Leftrightarrow [\forall i < j: W_{p(j)} \subset W_{p(i)} \Rightarrow W_{p(j)} \text{ is finite}]\\
\Caut_{\mathbf{Tar}}(p,T) &\Leftrightarrow [\forall i: \neg( \content(T) \subset W_{p(i)})] 
\end{align*}
\end{defn}

\setboolean{useproof}{true}

\pgfdeclarelayer{background}
\pgfdeclarelayer{foreground}
\pgfsetlayers{background,main,foreground}

\begin{figure}[ht]
\begin{center}
\begin{tikzpicture}[above,sloped,shorten <=2mm,, shorten >=2mm]

\begin{scope}[every node/.style={minimum size=4mm}]

\node (nothing) at (1,0) {\textbf{T}};
\node (caut) at (1,-3.5)    	{$\Caut$};

\node (cautinf) at (-1.5,-1) {$\Caut_\infty$};
\node (cauttar) at (1,-2) {$\Caut_\textbf{Tar}$};
\node (cautfin) at (3.5,-2) {$\Caut_\textbf{Fin}$};

\draw (caut) -- (cautfin);
\draw (caut) -- (cauttar);
\draw (caut) -- (cautinf);
\draw (cautinf) -- (nothing);
\draw (cauttar) -- (nothing);
\draw (cautfin) -- (nothing);

\draw[draw=lightgray, line width = 2pt,line cap=round] (-3,-1.2) -- (5,-1.2);

\end{scope}

\end{tikzpicture}
\end{center}
\caption{Relation of different variants of cautious learning. A black line indicates inclusion (bottom to top); all and only the black lines meeting the gray line are proper inclusions.}\label{fig:GoldCautiousV}
\end{figure}
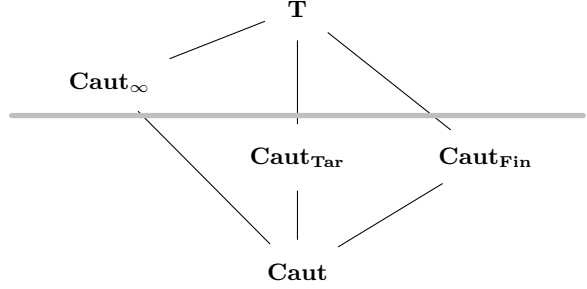

\setboolean{useproof}{true}

The proof of the following theorem is essentially the same as given in \cite{Osh-Sto-Wei:b:86:stl} to show that cautious learning is a proper restriction of $\Text\Gold\Ex$-learning, we now extend it to strongly decisive learning. Note that a different extension was given in \cite{Bal-Cas-Mer-Ste-Wie:j:08} (with an elegant proof exploiting the undecidability of the halting problem), pertaining to \emph{behaviorally correct} learning. The proof in \cite{Bal-Cas-Mer-Ste-Wie:j:08} as well as our proof would also carry over to the combination of these two extensions.

\begin{thm}\label{thm:ConvWMonInT}
There is a class of languages that is $\Text\Gold\SDec\Mon\Ex$-learnable, but not $\Text\Gold\Caut\Ex$-learnable.
\end{thm}
\begin{proof}
Let $h$ be a $\Psd$-learner as follows,
$$\forall D, t : h(D,t) = \varphi_{\max(D)}(t),$$
and $\CalL = \Text\Psd\SDec\Mon\Ex(h)$. Suppose $\CalL$ is $\Text\Gold\Caut\Ex$-learnable through learner $h' \in \CalR$. We define, for all $\sigma$ and $t$, the total computable predicate $Q(\sigma, t)$ as 
$$Q(\sigma,t) \Leftrightarrow \content(\sigma) \subset W_{h'(\sigma)}^t.$$

We let $\ind$ such that, for every set $D$, $W_{\ind(D)} = D$. Using $\ORT$ we define $p$ and $e \in \CalR$ strongly monotone increasing such that for all $n$ and $t$,
\begin{align*}
W_p &= \range(e);\\
\varphi_{e(n)} &= \begin{cases} \ind(\content(e[n+1])), & \text{if } Q(e[n+1],t); \\
p, & \text{otherwise.} \end{cases}
\end{align*}
\textit{Case 1:} For all $n$ and $t$, $Q(e[n+1],t)$ does not hold. Then we have $\varphi_{e(n)}(t) = p$ for all $n,t$. Thus $W_p \in \CalL$ as for any $D \subseteq W_p$, $h(D,t) = \varphi_{\max(D)}(t) = p$. But $h'$ does not $\Text\Gold\Caut\Ex$-learns $W_p$ from text $e$ as for all $n$ and $t$, $\content(e[n])$ is not a proper subset of $W_{h'(e[n])}$ in $t$ steps although $W_p$ is infinite.

\textit{Case 2:} There are $n$ and $t$ such that $Q(e[n+1],t)$ holds. Then we have $\content(e[n+1]) \in \CalL$ as we will show now. Let $T$ be a text for $\content(e[n+1])$. As $e$ is monotone increasing we have that $e(n)$ is the maximal element in $\content(e[n+1])$. Additionally, for all $t' \geq t$, we have  $\varphi_{e(n)}(t') = \varphi_{e(n)}(t) = \ind(\content(e[n+1]))$. As $h$ makes only one mind change the strongly decisive and monotone conditions hold. Thus, there is $n_0$ such that, for all $n \geq n_0$, $h(\content(T[n]),n) = h(\content(T[n_0]),n_0) = \ind(\content(e[n+1]))$, i.e.\ $\content(e[n+1]) \in \CalL$.

The learner $h'$ does not $\Text\Gold\Caut\Ex$-learn $\content(e[n+1])$ as we know from the predicate $Q$ that $\content(e[n+1]) \subset W_{h'(\content(e[n+1]))}$ and the cautious learner $h'$ must not change to a proper subset of a previous hypothesis. 
\end{proof}

\setboolean{useproof}{true}

The following theorem contradicts a theorem given as an exercise in \cite{Osh-Sto-Wei:b:86:stl} (Exercise 4.5.4B).

\begin{thm}\label{thm:CautVarConv}
For $\delta \in \{\Caut, \Caut_\mathbf{Tar}, \Caut_\mathbf{Fin}\}$ we have
$$[\Text\Gold\delta\Ex] = [\Text\Gold\Conv\Ex].$$
\end{thm}
\begin{proof}
We get the inclusion [\textbf{TxtGConvEx}] $\subseteq$ [\textbf{TxtGCautEx}] as a direct consequence from Theorem~\ref{thm:ConvInSDec}. Obviously we have $[\Text\Gold\Caut\Ex] \subseteq [\Text\Gold\Caut_\mathbf{Tar}\Ex]$ and $[\Text\Gold\Caut\Ex] \subseteq [\Text\Gold\Caut_\mathbf{Fin}\Ex]$. Thus, it suffices to show $[\Text\Gold\delta\Ex] \subseteq [\Text\Gold\Conv\Ex]$.

Let $\mathcal{L}$ be $\Text\Gold\delta\Ex$-learnable by a syntactically decisive learner $h \in \mathcal{R}$ (see Corollary~\ref{cor:SinkLocking}).
Using the S-m-n Theorem we get a function $p \in \mathcal{R}$ such that 
$$\forall \sigma: W_{p(\sigma)} = \bigcup_{t \in \mathbb{N}} 
\begin{cases}
W_{h(\sigma)}^t,			&\mbox{if }\forall \rho \in (W_{h(\sigma)}^t)^*, |\sigma \diamond \rho| \leq t: h(\sigma \diamond \rho) = h(\sigma);\\
\emptyset,						&\mbox{otherwise.}
\end{cases}$$
We let $Q$ be the following computable predicate.
$$
Q(\hat{\sigma},\sigma) \Leftrightarrow  h(\hat{\sigma}) \neq h(\sigma)\ \land\ \content(\sigma) \not\subseteq W_{h(\hat{\sigma})}^{|\sigma|-1}.
$$
For given sequences $\sigma$ and $\tau$ we say $\tau \preceq \sigma$ if
$$
\content(\tau) \subseteq \content(\sigma)\ \land\ |\tau| \leq |\sigma|.
$$
This means that, for every $\sigma$, the set of all $\tau$ such that $\tau \preceq \sigma$ is finite and computable. We define a learner $h'$ such that $h'(\sigma) = p(\hat{\sigma})$ where $\hat{\sigma} \preceq \sigma$ using recursion. For a given sequence $\sigma \neq \emptyset$ let $\hat{\sigma}$ be such that $h'(\sigma^-)=p(\hat{\sigma})$.
$$
\forall \sigma: h'(\sigma) = 
\begin{cases}
p(\emptyset),									&\mbox{if }\sigma = \emptyset;\\
p(\tau \diamond \sigma),			&\mbox{else, if }\exists \tau, \hat{\sigma} \subseteq \tau \preceq \sigma: Q(\hat{\sigma},\tau);\footnotemark\\
h'(\sigma^-),		&\mbox{otherwise.}
\end{cases}\footnotetext{We choose the least such $\tau$, if existent.}
$$
This means $h'$ only changes its hypothesis if $Q$ ensures that $h$ made a mind change and that the previous hypothesis does not contain something of the new input data. We first show that $h'$ is conservative. Let $\sigma$ and $\hat{\sigma}$ be such that $h'(\sigma^-) = p(\hat{\sigma})$ and let $\tau \preceq \sigma$ be such that $Q(\hat{\sigma},\tau)$. Then we have, for all $t \geq |\tau|$ with $\content(\tau) \subseteq W_{h(\sigma)}^t$, 
\begin{align*}
\neg [&\forall \rho \in (W_{h(\hat{\sigma})}^t)^*, |\hat{\sigma} \diamond \rho| \leq t: h(\hat{\sigma} \diamond \rho) = h(\hat{\sigma})], \text{ which is equivalent to} \\
&\exists \rho  \in (W_{h(\hat{\sigma})}^t)^*, |\hat{\sigma} \diamond \rho| \leq t: h(\hat{\sigma} \diamond \rho) \neq h(\hat{\sigma});
\end{align*}
as there is $\rho$ such that $\hat{\sigma} \diamond \rho = \tau$. Therefore, we get $\content(\tau) \nsubseteq W_{p(\hat{\sigma})}$, as $W_{h(\hat{\sigma})}^t$ is monotone non-decreasing in $t$. Thus, $h'$ is conservative.

Second, we will show that $h'$ converges on any text $T$ for a language $L \in \mathcal{L}$. Let $L \in \mathcal{L}$ and $T$ be a text for $L$. Thus, $h$ converges on $T$. Suppose $h'$ does not converge on $T$. Let $(p(\sigma_i))_{i \in \mathbb{N}}$ the corresponding sequence of hypotheses. Then $T' = \bigcup_{i \in \mathbb{N}} \sigma_i$ is a text for $L$ as for every $i \in \mathbb{N}$, $T(i) \in \content(\sigma_{i+1})$. As $h'$ infinitely often changes its mind, we have that, for infinitely many $\sigma_i$, there is, for each $i$, $\tau_i$ such that $\sigma_i \subseteq \tau_i \subseteq \sigma_{i+1}$ with $Q(\sigma_i,\tau_i)$ holds. As $Q(\sigma_i,\tau_i)$ means that $h(\sigma_i) \neq h(\tau_i)$, $h$ diverges on $T'$, a contradiction. 

Third we will show that $h'$ converges to a correct hypothesis. Let $\sigma$ be such that $h'$ converges to $p(\sigma)$ on $T$. In the following we consider two cases for this $\sigma$. \par
\textit{Case 1:} If $\sigma$ is a locking sequence of $h$ on $L$ we have, for all $\tau \in \seq(L)$, $h(\sigma \diamond \tau) = h(\sigma)$ and especially for all $\rho \in (W_{h(\sigma)}^t)^*$ with $|\sigma \diamond \rho| \leq t$,  $h(\sigma \diamond \rho) = h(\sigma)$. Thus, $W_{p(\sigma)} = W_{h(\sigma)} = L$. 

\textit{Case 2:} Suppose $\sigma$ is not a locking sequence. As $\content(T) = L$ and $h'$ converges, we have for all $n$ and $\tau$ with $\sigma \subseteq \tau \preceq T[n]$, $\neg Q(\sigma, \tau)$. This means that, for all $\tau$ with elements of $L$ and $\sigma \subseteq \tau,$ $\neg Q(\sigma,\tau)$, i.e. 
\begin{equation}\label{eq:NotQSigmaTau}
\forall \tau \in \seq(L) : h(\sigma) = h(\tau)\ \lor\ \content(\tau) \subseteq W_{h(\sigma)}^{|\tau|-1}. 
\end{equation}
We now show $L \subseteq W_{h(\sigma)}$. If we have, for all $\tau \in \seq(L)$, $h(\sigma) = h(\tau)$, we get this directly from Equation~(\ref{eq:NotQSigmaTau}). Otherwise, let $\tau$ be such that $h(\sigma) \neq h(\sigma \diamond \tau)$. Let $x \in L$. Thus, $h(\sigma) \neq h(\sigma \diamond \tau \diamond x)$, as $h$ is syntactically decisive. From $\neg Q(\sigma, \sigma \diamond \tau \diamond x)$ we can conclude that $\content(\sigma \diamond \tau \diamond x) \subseteq W_{h(\sigma)}^{|\sigma \diamond \tau \diamond x|}$. Therefore we have, for all $x \in L$, $x \in W_{h(\sigma)}$ and thus $\content(T)  = L \subseteq W_{h(\sigma)}$. 

Additionally we will show now that $W_{h(\sigma)} = W_{p(\sigma)}$. Obviously we have $W_{p(\sigma)} \subseteq W_{h(\sigma)}.$ To show that $W_{h(\sigma)} \subseteq W_{p(\sigma)}$ suppose there is $x \in W_{h(\sigma)}$ such that $x \notin W_{p(\sigma)}$. Then there is a minimal $t$ such that $x \in W_{h(\sigma)}^t$ but there is $\rho \in  (W_{h(\sigma)}^t)^*, |\sigma \diamond \rho| \leq t$ such that $h(\sigma \diamond \rho) \neq h(\sigma)$ and therefore $h(\sigma \diamond \rho \diamond x) \neq h(\sigma \diamond \rho).$ As we have $\neg Q(\sigma, \sigma \diamond \rho \diamond x)$ which is equivalent to $h(\sigma) = h(\sigma \diamond \rho \diamond x)\ \lor\ \content(\sigma \diamond \rho \diamond x) \subseteq W_{h(\sigma)}^{|\sigma \diamond \rho \diamond x| - 1}$ and we supposed that $h(\sigma \diamond \rho \diamond x) \neq h(\sigma)$ it follows that $\content(\sigma \diamond \rho \diamond x) \subseteq W_{h(\sigma)}^{|\sigma \diamond \rho \diamond x| - 1}.$ This is a contradiction as $|\sigma \diamond \rho \diamond x| -1 \leq t.$ Thus, for all $x \in L$ we have $x \in W_{p(\sigma)}$ and from $L \subseteq W_{h(\sigma)}$ we get $W_{h(\sigma)} \subseteq W_{p(\sigma)}$. 

(a) $\delta = \Caut.$ We have that the learner must not change to a proper subset of a previous hypothesis and this means that $W_{h(\sigma)} = L$. \par
(b) $\delta = \Caut_\mathbf{Tar}.$ The learner $h$ never returns a hypothesis which is a proper superset of the language that is learned. Thus $W_{h(\sigma)} = L$. \par 
(c) $\delta = \Caut_\mathbf{Fin}.$ As $h$ must not change to a finite subset of a previous hypothesis, we suppose that $W_{h(\sigma)} \supset L$ and both $W_{h(\sigma)}$ and $L$ are infinite. This means there is a sequence $\rho \in \seq(L) \subseteq \seq(W_{h(\sigma)})$ such that $h(\sigma) \neq h(\sigma \diamond \rho)$. Thus, $W_{p(\sigma)}$ is finite, a contradiction to $W_{h(\sigma)}$ being infinite. Therefore we have $W_{h(\sigma)} = L$.
\end{proof}

\setboolean{useproof}{true}

From the definitions of the learning criteria we have $[\Text\Gold\Conv\Ex] \subseteq [\Text\Gold\WMon\Ex]$. Using Theorem~\ref{thm:CautVarConv} and the equivalence of weakly monotone and conservative learning (using $\Gold$) \cite{Kin-Ste:j:95:mon,Jai-Sha:j:98}, we get the following.

\begin{cor}\label{cor:ConvWMonCaut}
We have
$$[\Text\Gold\Conv\Ex] = [\Text\Gold\WMon\Ex] = [\Text\Gold\Caut\Ex].$$
\end{cor}

\setboolean{useproof}{true}

Using Corollary~\ref{cor:ConvWMonCaut} and Theorem~\ref{thm:ConvInSDec} we get that weakly monotone $\Text\Gold\Ex$-learning is included in strongly decisive $\Text\Gold\Ex$-learning. Theorem~\ref{thm:ConvWMonInT} shows that this inclusion is proper.

\begin{cor}\label{cor:WMonInSDec}
We have $$[\Text\Gold\WMon\Ex] \subset [\Text\Gold\SDec\Ex].$$
\end{cor}

\setboolean{useproof}{true}

The next theorem is the last theorem of this section and shows that forbidding to go down to strict \emph{infinite} subsets of previously conjectures sets is no restriction.

\begin{thm}\label{thm:CautInfT}
We have 
$$[\Text\Gold\Caut_\infty\Ex] = [\Text\Gold\Ex].$$
\end{thm}
\begin{proof}
Obviously we have $[\Text\Gold\Caut_{\infty}\Ex] \subseteq [\Text\Gold\Ex]$. Thus, we have to show that $[\Text\Gold\Ex] \subseteq [\Text\Gold\Caut_\mathbf{\infty}\Ex]$.
Let $\CalL$ be a set of languages and $h$ be a learner such that $h$ $\Text\Gold\Ex$-learns $\CalL$ and $h$ is strongly locking on $\CalL$ (see Corollary~\ref{cor:SinkLocking}). We define, for all $\sigma$ and $t$, the set $M_{\sigma}^t$ such that
$$M_{\sigma}^t = \{\tau\ |\ \tau \in \seq(W_{h(\sigma)}^t \cup \content(\sigma))\ \land\ |\tau \diamond \sigma| \leq t\}.$$
Using the S-m-n Theorem we get a function $p \in \mathcal{R}$ such that
$$\forall \sigma: W_{p(\sigma)} = \content(\sigma) \bigcup_{t \in \mathbb{N}} 
\begin{cases}
W_{h(\sigma)}^t,			&\mbox{if }\forall \rho \in M_{\sigma}^t: h(\sigma \diamond \rho) = h(\sigma);\\
\emptyset,						&\mbox{otherwise.}
\end{cases}$$
We define a learner $h'$ as
$$\forall \sigma : h'(\sigma) = \begin{cases} p(\sigma), & \text{if } h(\sigma) \neq h(\sigma^-); \\ h'(\sigma^-), & \text{otherwise.} \end{cases}$$
We will show now that the learner $h'$ $\Text\Gold\Caut_{\infty}\Ex$-learns $\CalL$. Let an $L \in \CalL$ and a text $T$ for $L$ be given. As $h$ is strongly locking there is $n_0$ such that for all $\tau \in \seq(L)$, $h(T[n_0] \diamond \tau) = h(T[n_0])$ and $W_{h(T[n_0])} = L$. Thus we have, for all $n \geq n_0$, $h'(T[n]) = h'(T[n_0])$ and $W_{h'(T[n_0])} = W_{p(T[n_0])} = W_{h(T[n_0])} = L$. To show that the learning restriction $\Caut_{\infty}$ holds, we assume that there are $i < j$ such that $W_{h'(T[j])} \subset W_{h'(T[i])}$ and $W_{h'(T[j])}$ is infinite.
W.l.o.g. $j$ is the first time that $h'$ returns the hypothesis $W_{h'(T[j])}$. Let $\tau$ be such that $T[i] \diamond \tau = T[j]$. From the definition of the function $p$ we get that $\content(T[j]) \subseteq W_{h'(T[j])} \subseteq W_{h'(T[i])}$. Thus, $\content(\tau) \subseteq W_{h'(T[i])} = W_{p(T[i])}$ and therefore $W_{p(T[i])}$ is finite, a contradiction to the assumption that $W_{h'(T[j])}$ is infinite. 
\end{proof}

\section{Decisiveness}
\setboolean{useproof}{true}

\label{sec:Decisiveness}

In this section the goal is to show that decisive and strongly decisive learning separate (see Theorem~\ref{thm:StronglyDecisiveLearning}). For this proof we adapt a technique known in computability theory as a ``priority argument'' (note, though, that we are not dealing with oracle computations). In order to illustrate the proof with a simpler version, we first reprove that decisiveness is a restriction to $\Text\Gold\Ex$-learning (as shown in \cite{Bal-Cas-Mer-Ste-Wie:j:08}).

\setboolean{useproof}{true}

For both proofs we need the following lemma, a variant of which is given in \cite{Bal-Cas-Mer-Ste-Wie:j:08} for the case of decisive learning; it is easy to see that the proof from \cite{Bal-Cas-Mer-Ste-Wie:j:08} also works for the cases we consider here.

\begin{lem}\label{lem:NotNatnum}
Let $\CalL$ be such that $\natnum \not\in \CalL$ and, for each finite set $D$, there are only finitely many $L \in \CalL$ with $D \not\subseteq L$. Let $\delta \in \{\Dec,\SDec\}$. Then, if $\CalL$ is $\Text\Gold\delta\Ex$-learnable, it is so learnable by a learner which never outputs an index for $\natnum$.
\end{lem}

\setboolean{useproof}{true}

Now we get to the theorem regarding decisiveness. Its proof is an adaptation of the proof given in \cite{Bal-Cas-Mer-Ste-Wie:j:08}, rephrased as a priority argument. This rephrased version will be modified later to prove the separation of decisive and strongly decisive learning.

\begin{thm}[\cite{Bal-Cas-Mer-Ste-Wie:j:08}]\label{thm:DecisiveLearning}
We have
$$
[\Text\Gold\Dec\Ex] \subset [\Text\Gold\Ex].
$$
\end{thm}
\begin{proof} For this proof we will employ a technique from computability theory known as \emph{priority argument}. For this technique, one has a set of \emph{requirements} (we will have one for each $e \in \natnum$) and a \emph{priority} on requirements (we will prioritize smaller $e$ over larger). One then tries to fulfill requirements one after the other in an iterative manner (fulfilling the unfulfilled requirement of highest priority without violating requirements of higher priority) so that, in the limit, the entire infinite list of requirements will be fulfilled.

We apply this technique in order to construct a learner $h \in \CalP$ (and a corresponding set of learned sets $\CalL = \Text\Gold\Ex(h)$). Thus, we will give requirements which will depend on the $h$ to be constructed. In particular, we will use a list of requirement $(R_e)_{e \in \natnum}$, where lower $e$ have higher priority. For each $e$, $R_e$ will correspond to the fact that learner $\varphi_e$ is not a suitable decisive learner for $\CalL$. We proceed with the formal argument.

For each $e$, let Requirement $R_e$ be the disjunction of the following three predicates depending on the $h$ to be constructed.
\begin{enumerate}[(i)]
	\item $\exists x$: $\forall \sigma \in \seq(\natnum \setminus \{x\}): \varphi_e(\sigma)\diverges \vee W_{\varphi_e(\sigma)} \neq \natnum \setminus \{x\}$ and $h$ learns $\natnum \setminus \{x\}$.
	\item $\exists \sigma \in \seq{}: \content(\sigma) \subset W_{\varphi_e(\sigma)}$ and $h$ learns $W_{\varphi_e(\sigma)}$ and some $D$ with $\content(\sigma) \subseteq D \subset W_{\varphi_e(\sigma)}$.
	\item $\exists \sigma \in \seq: W_{\varphi_e(\sigma)} = \natnum$.
\end{enumerate}
If all $(R_e)_{e \in \natnum}$ hold, then every learner which never outputs an index for $\natnum$ fails to learn $\CalL$ decisively as follows. For each learner $\varphi_e$ which never outputs an index for $\natnum$, either (i) of $R_e$ holds, implying that some co-singleton is learned by $h$ but not by $\varphi_e$. Or (ii) holds, then there is a $\sigma$ on which $\varphi_e$ generalizes, but will later have to abandon this correct conjecture $p = \varphi_e(\sigma)$ in order to learn some finite set $D$; as, after the change to a hypothesis for $D$, the text can still be extended to a text for $W_p$, the learner is not decisive.\footnote{One might wonder why the U-shape can be achieved on a language which is to be learned: after all, those can be avoided, according to the theorem that non-U-shaped learning is not a restriction to $\Text\Gold\Ex$ \cite{Bal-Cas-Mer-Ste-Wie:j:08}. However, the price for avoiding it is to output a conjecture for $\natnum$.}

Thus, all that remains is to construct $h$ in a way that all of $(R_e)_{e \in \natnum}$ are fulfilled. In order to coordinate the different requirements when constructing $h$ on different inputs, we will divide the set of all possible input sequences into infinitely many segments, of which every requirement can ``claim'' up to two at any point of the algorithm defining $h$; the chosen segments can change over the course of the construction, and requirements of higher priority might ``take away'' segments from requirements with lower priority (but not vice versa). We follow~\cite{Bal-Cas-Mer-Ste-Wie:j:08} with the division of segments: For any set $A \subset \natnum$ we let $\id(A) = \min(\natnum \setminus A)$ be the \emph{ID of $A$}; for ease of notation, for each finite sequence $\sigma$, we let $\id(\sigma) = \id(\content(\sigma))$. For each $s$, the $s$th segment contains all $\sigma$ with $\id(\sigma) = s$. We note that $\id$ is \emph{monotone}, i.e.
\begin{equation}\label{eq:IDMonotone}
\forall A,B \subset \natnum: A \subseteq B \Rightarrow \id(A) \leq \id(B).
\end{equation}
The first way of ensuring some requirement $R_e$ is via (i); as this part itself is not decidable, we will check a ``bounded'' version thereof. We define, for all $e,t,s$,
$$
P_{e,t}(s)  \Leftrightarrow (\forall \sigma \in \seq_{\leq t} \mid \id(\sigma) = s) \; \Phi_e(\sigma) > t \vee \content(\sigma) \not\subset W_{\varphi_e(\sigma)}^t.
$$
For any $e$, if we can find an $s$ such that, for all $t$, we have $P_{e,t}(s)$, then it suffices to make $h$ learn $\natnum \setminus \{s\}$ in order to fulfill $R_e$ via part (i); this requires control over segment $s$ in defining $h$. 

Note that, if we ever cannot take control over some segment because some requirement with higher priority is already in control, then we will try out different $s$ (only finitely many are blocked).

If we ever find a $t$ such that $\neg P_{e,t}(s)$, then we can work on fulfilling $R_e$ via (ii), as we directly get a $\sigma$ where $\varphi_e$ over the content generalizes. In order to fulfill $R_e$ via (ii) we have to choose a finite set $D$ with $\content(\sigma) \subseteq D \subset W_{\varphi_e(\sigma)}$. We will then take control over the segments corresponding to $\id(D)$ and $\id(W_{\varphi_e(\sigma)}^t)$ (for growing $t$), \emph{but not necessarily over segment $s$}, and thus establish $R_e$ via (ii). Note that, again, the segments we desire might be blocked; but only finitely many are blocked, and we require control over $\id(D)$ and $\id(W_{\varphi_e(\sigma)}^t)$, both of which are at least $s$ (this follows from $\id$ being monotone, see Equation~(\ref{eq:IDMonotone}), and from $\content(\sigma) \subseteq D \subset W_{\varphi_e(\sigma)}^t$); thus, we can always find an $s$ for which we can either follow our strategy for (i) or for (ii) as just described.

It is tempting to choose simply $D = \content(\sigma)$, this fulfills all desired properties. The main danger now comes from the possibility of $\varphi_e(\sigma)$ being an index for $\natnum$: this will imply that, for growing $t$, $y = \id(W_{\varphi_e(\sigma)}^t)$ will also be growing indefinitely. Of course, there is no problem with satisfying $R_e$, it now holds via (iii); but as soon as at least two requirements will take control over segments $y$ for indefinitely growing $y$, they might start blocking each other (more precisely, the requirement of higher priority will block the one of lower priority). We now need to know something about our later analysis: we will want to make sure that every requirement $R_e$ either (a) converges in which segments to control or (b) for all $n$, there is a time $t$ in the definition of $h$ after which $R_e$ will never have control over any segment corresponding to IDs $\leq n$; in fact, we will show this later by induction (see Claim~\ref{claim:InductionProof}). Any requirement which takes control over segments $y$ for indefinitely growing $y$ might be blocked infinitely often, and thus forced to try out different $s$ for fulfilling $R_e$, including returning to $s$ that were abandoned previously because of (back then) being blocked by a requirement of higher priority. Thus, such a requirement would fulfill neither (a) nor (b) from above. We will avoid this problem by \emph{not} choosing $D = \content(\sigma)$, but instead choosing a $D$ which grows in ID along with the corresponding $W_{\varphi_e(\sigma)}^t$. The idea is to start with $D = \content(\sigma)$ and then, as $W_{\varphi_e(\sigma)}^t$ grows, add more elements. For this we make some definitions as follows.

For a finite sequence $\sigma$ we let $\id'(\sigma)$ be the least element not in $\content(\sigma)$ which is larger than all elements of $\content(\sigma)$. For any finite sequence $\sigma$ and $e,t \geq 0$ we let $D^t_{e,\sigma}$ be such that
$$
D^t_{e,\sigma} = 
\begin{cases}
\content(\sigma),			&\mbox{if }\id(W_{\varphi_e(\sigma)}^t) \leq \id'(\sigma);\\
\{0,\ldots, \id(W_{\varphi_e(\sigma)}^t)-2\},			&\mbox{otherwise.}
\end{cases}
$$
For all $e,t$ and $\sigma$ with $\content(\sigma) \subset W_{\varphi_e(\sigma)}$ we have
\begin{equation}\label{eq:defDTSigma}
\content(\sigma) \subseteq D^t_{e,\sigma} \subset W_{\varphi_e(\sigma)}.
\end{equation}
Thus, we will use the sets $D^t_{e,\sigma}$ to satisfy (ii) of $R_e$ (in place of $D$).

We now have all parts that are required to start giving the construction for $h$. In that construction we will make use of a subroutine which takes as inputs a set $B$ of blocked indices, a requirement $e$ and a time bound $t$, and which finds triples $(x,y,\sigma)$ with $x,y \not\in B$ such that
\begin{equation}\label{eq:defnTWitness}
P_{e,t}(x) \mbox{ or } \big[\content(\sigma) \subset W_{\varphi_e(\sigma)}^t \wedge \id(D^t_{e,\sigma}) = x \wedge \id(W_{\varphi_e(\sigma)}^t) = y\big].
\end{equation}
We call $(x,y,\sigma)$ fulfilling Equation~(\ref{eq:defnTWitness}) for given $t$ and $e$ a \emph{$t$-witness for $R_e$}. The subroutine is called \findWitness\ and is given in Algorithm~\ref{alg:priorityArgumentDecSubroutine}.
\begin{algorithm}

\For{$s = 0$ \To $\max(B) + 1$}{
	\uIf{$P_{e,t}(s)$ and $s \not\in B$}{
		\Return $(s,s,0)$\;
	} \ElseIf{\label{line:PCondition}$\neg P_{e,t}(s)$} {
		Let \label{line:sigmaSearch} $\sigma$ be minimal with $\id(\sigma) = s$ and $\content(\sigma) \subset W_{\varphi_e(\sigma)}^t$\;
		$x$ $\assign$ $\id(D^t_{e,\sigma})$\;
		$y$ $\assign$ $\id(W_{\varphi_e(\sigma)}^t)$\;
		\If{$x \not\in B$ and $y \not\in B$}{
			\Return $(x,y,\sigma)$\;
		}
	}
}
\Return \texttt{error}\;
\caption{\findWitness$(B,e,t)$}\label{alg:priorityArgumentDecSubroutine}
\end{algorithm}

We now formally show termination and correctness of our subroutine.
\begin{claim} Let $e,t$ and a finite set $B$ be given. The algorithm \findWitness\ on $(B,e,t)$ terminates and returns a $t$-witness $(x,y,\sigma)$ for $R_e$ such that $x,y \not\in B$.
\end{claim}
\begin{claimProof}
From the condition in line~\ref{line:PCondition} we see that the search in line~\ref{line:sigmaSearch} is necessarily successful, showing termination.
Using the monotonicity of $\id$ from Equation~(\ref{eq:IDMonotone}) on Equation~(\ref{eq:defDTSigma}) we have that the subroutine \findWitness\ cannot return \texttt{error} on any arguments $(B,e,t)$: for $s=\max(B)+1$, we either have $P_{e,t}(s)$ or the $x$ and $y$ chosen are larger than $\id(\sigma) = s > \max(B)$.
\end{claimProof}

With the subroutine given above, we now turn to the priority construction for defining $h$ detailed in Algorithm~\ref{alg:priorityArgumentDec}. This algorithm assigns witness tuples to more and more requirements, trying to make sure that they are $t$-witnesses, for larger and larger $t$. For each $e$, $w_e(t)$ will be the witness tuple associated with $R_e$ after $t$ iterations (defined for all $t \geq e$). We say that a requirement $R_e$ \emph{blocks} an ID $n$ iff $n \in \{x,y\}$ for the witness tuple $w_e(t) = (x,y,\sigma)$ currently associated with $R_e$. We say that a tuple $(x,y,\sigma)$ is \emph{$(e,t)$-legal} iff it is a $t$-witness for $R_e$ and $x$ and $y$ are not blocked by any $R_{e'}$ with $e' < e$. Clearly, it is decidable whether a triple is $(e,t)$-legal.

In order to define the learner $h$ we will need some functions giving us indices for the languages to be learned. To that end, let $p,q \in \CalR$ (using the S-m-n Theorem) be such that
\begin{eqnarray*}
\forall n: W_{q(n)} & = & \natnum \setminus \{n\};\\
\forall e,t,\sigma: W_{p(e,t,\sigma)} & = & D^t_{e,\sigma}.
\end{eqnarray*}
To increase readability, we allow assignments to values of $h$ for arguments on which $h$ was already defined previously; in this case, the new assignment has no effect.
\begin{algorithm}
\For{$t = 0$ \To $\infty$}{
	\For{$e = 0$ \To $t$}{
		\uIf{$t=0$, $w_e(t-1)$ is undefined or $w_e(t-1)$ is not $(e,t)$-legal}{
			Let $B$ be the set of IDs blocked by any $e' < e$\;
			$(x,y,\sigma)$ $\assign$ \findWitness$(B,e,t)$\;
		} \Else {
			$(x,y,\sigma)$ $\assign$ $w_e(t-1)$\;
		}
		$w_e(t)$ $\assign$ $(x,y,\sigma)$\;
		\uIf{$P_{e,t}(x)$}{
			\ForEach{$\tau \in \seq_{\leq t}$ with $\id(\tau) = x$}{
					$h(\tau)$ $\assign$ $q(x)$\;
			}
		}
		\Else{
			\ForEach{$\tau \in \seq_{\leq t}$ with $\content(\tau) = D^t_{e,\sigma}$}{
					$h(\tau)$\label{line:outputP} $\assign$ $p(e,t,\sigma)$\;
			}
			\ForEach{$\tau \in \seq_{\leq t}$ with $\id(\tau) = y$}{
					$h(\tau)$\label{line:folowE} $\assign$ $\varphi_e(\sigma)$\;
			}
		}
	}
}
\caption{Priority Construction $\Dec$}\label{alg:priorityArgumentDec}
\end{algorithm}
Regarding Algorithm~\ref{alg:priorityArgumentDec}, note that lines 3--8 make sure that we have an appropriate witness tuple. We will later show that the sequence of assigned witness tuples will converge (for learners never giving a conjecture for $\natnum$). Lines 9--11 will try to establish the requirement $R_e$ via (i), once this fails it will be established in lines 12--16 via (ii). 

After this construction of $h$, we let $\CalL = \Text\Gold\Ex(h)$ be the target to be learned. First note that the IDs blocked by different requirements are always disjoint (at the end of an iteration of $t$). As the major part of the analysis, we show the following claim by induction, showing that, for each $e$, either the triple associated with $R_e$ converges or it grows arbitrarily in both its $x$ and $y$ value (this is what we earlier had to carefully choose the $D$ for).

\begin{claim}\label{claim:InductionProof}
For all $e$ we have $R_e$ and, for all $n$, there is $t_0$ such that either
$$
\forall t \geq t_0: R_e \mbox{ does not block any ID }\leq n
$$
or
$$
\forall t \geq t_0: w_e(t) = w_e(t_0).
$$
\end{claim}
\begin{claimProof}
As our induction hypothesis, let $e$ be given such that the claim holds for all $e' < e$.

Case 1: There is $t_0$ such that $\forall t \geq t_0: w_e(t) = w_e(t_0)$.\\
Then, for all $t$, $(x,y,\sigma) = w_e(t_0)$ is a $t$-witness for $R_e$; in the case of $\forall t: P_{e,t}(x)$, we have that, for all but finitely many $\tau$ with $\id(\tau) = x$, $h(\tau) = q(x)$, and index for $\natnum \setminus \{x\}$; this implies $\natnum \setminus \{x\} \in \CalL$, which shows $R_e$. 

Otherwise we have, for all $t \geq t_0$, $D^t_{e,\sigma} = D^{t_0}_{e,\sigma}$. Furthermore we get, for all but finitely many $\tau$ with $\content(\tau) = D^{t_0}_{e,\sigma}$, $h(\tau) = p(e,t,\sigma)$, and index for $D^{t_0}_{e,\sigma}$; this implies $D^{t_0}_{e,\sigma} \in \CalL$. Consider now all those $\tau$ with $\id(\tau) = y$. If $\id(D^{t_0}_{e,\sigma}) = y$, then $h$ is already be defined on infinitely many such $\tau$, namely in case of $\content(\tau) = D^{t_0}_{e,\sigma}$. However, we have that $D^{t_0}_{e,\sigma}$ is a \emph{proper} subset of $W_{\varphi_e(\sigma)}$, which shows that, on any text for $W_{\varphi_e(\sigma)}$, $h$ will eventually only output $\varphi_e(\sigma)$, which gives $W_{\varphi_e(\sigma)} \in \CalL$ as desired and, thus, $R_e$.

Case 2: Otherwise.\\
For each ID $s$ there exists at most finitely many $\sigma$ with $\id(\sigma) = s$ and $\sigma$ is used in the witness triple for $R_e$; this follows from the choice of $\sigma$ in the subroutine \findWitness\ as a minimum, where, for larger $t$, all previously considered $\sigma$ are still considered (so that the chosen minimum might be smaller for larger $t$, but never go up, which shows convergence). A triple is only abandoned if it is not legal any more; this means it is either blocked or it is not a $t$-witness triple for some $t$. Using the induction hypothesis, the first can only happen finitely many times for any given tuple; the second implies the desired increase in both the $x$ and the $y$ value of the witness tuple. For this we also use our specific choice of $D$ as growing along with the ID of the associated $W_{\varphi_e(\sigma)}^t$ and we use that any witness tuple with a $\sigma$ with $\id(\sigma) = s$ has $x$ and $y$ value of at least $s$, due to the monotonicity of $\id$.

To show $R_e$ (we will show (3)), let $t_1$ be the maximum over all $t_0$ existing for the converging $e' < e$ by the induction hypothesis and $e$. Let $(x,y,\sigma) = w_e(t_1)$ be the $t_1$-witness triple chosen for $R_e$ in iteration $t_1$. Suppose, by way of contradiction, that $\varphi_{e}(\sigma)$ is not an index for $\natnum$; let $n = \id(W_{\varphi_{e}(\sigma)})$. Let $t_2$ be the maximum over all $t_0$ found by the induction hypothesis for all $e' < e$ with the chosen $n$. Since the triple $(x,y,\sigma)$ is $(e,t)$-legal for all $t \geq t_2$, we get a contradiction to the unbounded growth of the witness triple.

This shows that $\varphi_{e}(\sigma)$ is an index for $\natnum$, and thus we have $R_e$.
\end{claimProof}
With the last claim we now see that all requirement are satisfied. This implies that $\CalL$ cannot be $\Text\Gold\Dec\Ex$-learned by a learner never using an index for $\natnum$ as conjecture. 

We have that $\natnum \not\in \CalL$. Furthermore, for any ID $s$, there are only finitely many sets in $\CalL$ with that ID; this implies that, for every finite set $D$, there are only finitely many elements $L \in \CalL$ with $D \not\subseteq L$. Thus, using Lemma~\ref{lem:NotNatnum}, $\CalL$ is not decisively learnable at all.
\end{proof}

\setboolean{useproof}{true}

While the previous theorem showed that decisiveness poses a restriction on $\Text\Gold\Ex$-learning, the next theorem shows that the requirement of strong decisiveness is even more restrictive. The proof follows the proof of Theorem~\ref{thm:DecisiveLearning}, with some modifications.

\begin{thm}\label{thm:StronglyDecisiveLearning}
We have
$$
[\Text\Gold\SDec\Ex] \subset [\Text\Gold\Dec\Ex].
$$
\end{thm}
\begin{proof}
We use the same language and definitions as in the proof of Theorem~\ref{thm:DecisiveLearning}.
The idea of this proof is as follows. We build a set $\CalL$ with a priority construction just as in the proof of Theorem~\ref{thm:DecisiveLearning}, the only essential change being in the definition of the hypothesis $p(e,t,\sigma)$: the change from $\varphi_e(\sigma)$ to $p(e,t,\sigma)$ and back to $\varphi_e(\sigma)$ on texts for $W_{\varphi_e(\sigma)}$ is what made $\CalL$ not decisively learnable. Thus, we will change $p(e,t,\sigma)$ to be a hypothesis for $W_{\varphi_e(\sigma)}$ as well -- \emph{as soon as $\varphi_e$ changed its hypothesis on an extension of $\sigma$}, and otherwise it is a hypothesis for $D_{e,\sigma}^t$ as before. This will make $h$ decisive on texts for $W_{\varphi_e(\sigma)}$, but $\varphi_e(\sigma)$ will not be strongly decisive.

Furthermore, we will make sure that for sequences with ID $s$, only conjectures for sets with ID $s$ are used, so that indecisiveness can only possibly happen within a segment. Now the last source of $\CalL$ not being decisively learnable is as follows. When different requirements take turns with being in control over the segment, they might introduce returns to abandoned conjectures. To counteract this, we make sure that any conjecture which is ever abandoned on a segment of ID $s$ is for $\natnum \setminus \{s\}$, which will give decisiveness.

We first define an alternative $p'$ for the function $p$ from that proof with the S-m-n Theorem such that, for all $e,t,\sigma$,
$$
W_{p'(e,t,\sigma)} = 
\begin{cases}
W_{\varphi_e(\sigma)},		&\mbox{if }\exists \tau \mbox{ with }\content(\tau) \subseteq D_{e,\sigma}^t: \varphi_e(\sigma \diamond \tau) \converges \neq \varphi_e(\sigma);\\
D_{e,\sigma}^t,						&\mbox{otherwise.}
\end{cases}
$$
As we have $D_{e,\sigma}^t \subseteq W_{\varphi_e(\sigma)}$, this is a valid application of the S-m-n Theorem.
We also want to replace the output of $h$ according to line~\ref{line:folowE} of Algorithm~\ref{alg:priorityArgumentDec}. To that end, let $g \in \CalR$ be as given by the S-m-n Theorem such that, for all $e$ and $\sigma$,
$$
W_{g(e,\sigma,y)} = W_{\varphi_e(\sigma)} \setminus \{y\}.
$$

We construct now a learner $h$ again according to a priority construction, as given in Algorithm~\ref{alg:priorityArgumentSDec}. Note that lines 1--\ref{line:ElseLine} are identical with the construction from Algorithm~\ref{alg:priorityArgumentDec} and lines 3--8 again make sure that we have an appropriate witness tuple and lines 9--11 try to establish the requirement $R_e$ via (i). The main difference lies in the way that $R_e$ is established once this fails in lines 12--18 via (ii): Here we need to check for a mind change and adjust what language $h$ should learn accordingly. 

\begin{algorithm}
\For{$t = 0$ \To $\infty$}{
	\For{$e = 0$ \To $t$}{
		\uIf{$t=0$, $w_e(t-1)$ is undefined or $w_e(t-1)$ is not $(e,t)$-legal}{
			Let $B$ be the set of IDs blocked by any $e' < e$\;
			$(x,y,\sigma)$ $\assign$ \findWitness$(B,e,t)$\;
		} \Else {
			$(x,y,\sigma)$ $\assign$ $w_e(t-1)$\;
		}
		$w_e(t)$ $\assign$ $(x,y,\sigma)$\;
		\uIf{$P_{e,t}(x)$}{
			\ForEach{$\tau \in \seq_{\leq t}$ with $\id(\tau) = x$}{
					$h(\tau)$ $\assign$ $q(x)$\;
			}
		}
		\Else{\label{line:ElseLine}
			\uIf{$\exists \tau \in \seq_{\leq t}(D^t_{e,\sigma}): \varphi_e(\sigma \diamond \tau)\converges_t \neq \varphi_e(\sigma)$}{
				\ForEach{$\tau \in \seq_{\leq t}$ with $\id(\tau) = y$}{
					$h(\tau)$\label{line:folowE2} $\assign$ $g(e,\sigma,y)$\;
				}
			} \Else{
				\ForEach{$\tau \in \seq_{\leq t}$ with $\content(\tau) = D^t_{e,\sigma}$}{
					$h(\tau)$\label{line:outputP2} $\assign$ $p'(e,t,\sigma)$\;
				}

			}
		}
	}
}
\caption{Priority Construction $\SDec$}\label{alg:priorityArgumentSDec}
\end{algorithm}

It is easy to check that $h$, on any sequence $\sigma$, gives conjectures for languages of the same ID as that of $\sigma$. Thus, indecisiveness of $h$ can only occur within a segment.

Next we will modify $h$ to avoid indecisiveness from different requirements taking turns controlling the same segment.
\ignore{
 To that end, we say that a hypothesis $a$ \emph{can be followed by $b$} iff
\begin{equation}\label{eq:allowedFollowing}
\exists e,t,\sigma: [a = p'(e,t,\sigma) \wedge b = g(e,\sigma)] \vee a=b.
\end{equation}
Note that this relation between hypotheses is decidable, thanks to the range of $p'$ and $g$ being decidable (and $p'$ and $g$ being 1-1). Intuitively, we only allow mind changes from $p'(e,t,\sigma)$ to $g(e,\sigma)$, but no other.
}
With the S-m-n Theorem we let $f \in \CalR$ be such that, for all $\sigma$,
$$
W_{f(\sigma)} = 
\begin{cases}
\natnum \setminus \{\id(\sigma)\},						&\mbox{if }\exists \tau \mbox{ with }\id(\sigma) \not\in \content(\tau):
h(\sigma) \neq h(\sigma \diamond \tau);\\
W_{h(\sigma)},			&\mbox{otherwise.}
\end{cases}
$$
Let $h'$ be such that, for all $\sigma$, 
$$
h'(\sigma) = 
\begin{cases}
h'(\sigma^-),			&\mbox{if }\sigma \neq \emptyset \mbox{ and } h(\sigma) = h(\sigma^-);\\
f(\sigma),				&\mbox{otherwise.}
\end{cases}
$$
We now let $\CalL = \Text\Gold\Dec\Ex(h')$. It is easy to see that $h'$ is decisive on all texts where it always makes an output, since indecisiveness can again only happen within a segment, and $f$ \emph{poisons} any possible non-final conjectures within a segment. 

Let a strongly decisive learner $\overline{h}$ for $\CalL$ be given which never makes a conjecture for $\natnum$ (we are reasoning with Lemma~\ref{lem:NotNatnum} again). Let $e$ be such that $\varphi_e = \overline{h}$. Reasoning as in the proof of Theorem~\ref{thm:DecisiveLearning}, we see that there is a triple $(x,y,\sigma)$ such that $w_e$ converges to that triple in the construction of $h'$. If, for all $t$, $P_{e,t}(x)$, then we have that $\natnum \setminus \{x\} \in \CalL$ (on any sequences with ID $x$, $h'$ gives an output for $\natnum \setminus \{x\}$, and it converges). Assume now that there is $t_0$ such that, for all $t \geq t_0$, we have $\neg P_{e,t}(x)$.

Case 1: There is $\tau$ with $\content(\tau) \subseteq D^t_{e,\sigma}$ such that $\varphi_e(\sigma \diamond \tau) \neq \varphi_e(\sigma)$.\\
Let $T$ be a text for $L = W_{\varphi_e(\sigma)}$. Then $h'$ on $T$ converges to an index for $L$, giving $L \in \CalL$. But this shows that $\overline{h} = \varphi_e$ was not strongly decisive on any text for $L$ starting with $\sigma \diamond \tau$, a contradiction.

Case 2: Otherwise.\\
Let $T$ be a text for $L = D^t_{e,\sigma}$. Then $h'$ on $T$ converges to an index for $L$, giving $L \in \CalL$. But $\overline{h} = \varphi_e$ converges on any text for $L$ starting with $\sigma$ to $\varphi_e(\sigma)$, a contradiction to $D^t_{e,\sigma} \subset W_{\varphi_e(\sigma)}$ (so the convergence is not to a correct hypothesis).

In both cases we get the desired contradiction.
\end{proof}

\section{Set-driven Learning}
\setboolean{useproof}{true}

\label{sec:SetDriven}

In this section we give theorems regarding set-driven learning. For this we build on the result that set-driven learning can always be done conservatively \cite{Kin-Ste:j:95:mon}.

\setboolean{useproof}{true}

First we show that any conservative set-driven learner can be assumed to be cautious and syntactically decisive, an important technical lemma.

\begin{lem}\label{thm:SdSyntDec}
We have  $$[\Text\Sd\Ex] = [\Text\Sd\Conv\SynDec\Ex].$$ 
In other words, every set-driven learner can be assumed syntactically decisive.
\end{lem}
\begin{proof}
Let a set-driven learner $h$ be given. Following \cite{Kin-Ste:j:95:mon} we can $h$ assume to be conservative. We define a learner $h'$ such that, for all finite sets $C$,
\begin{align*}
h'(C) = \begin{cases} \text{pad}(h(C),0), & \text{if } \forall D \subseteq C : h(D) = h(C) \rightarrow\\
&\;\;\; \forall D', D \subseteq D' \subseteq C : h(D') = h(D); \\
\pad(h(C),|C|+1), & \text{otherwise.} \end{cases}
\end{align*}
Let $\CalL = \Text\Sd\Conv\Ex(h)$. We will show that $h'$ is syntactically decisive and $\Text\Sd\Conv\Ex$-learns $\CalL$. Let $L \in \CalL$ be given and let $T$ be a text for $L$. First, we show that $h'$ $\Text\Ex$-learns $L$ from $T$. As $h$ is a set driven learner there is $n_0$ such that $\forall n \geq n_0 : h(\content(T[n_0])) = h(\content(T[n]))$ and $W_{h(\content(T[n_0]))} = L$. We will show that, for all $T[n]$ with $n \geq n_0$, the first condition in the definition of $h'$ holds. Let $n \geq n_0$ and suppose there are $D$ and $D'$ with 
\begin{align*}
D &\subseteq \content(T[n]), \\
h(D) &= h(\content(T[n])) = h(\content(T[n_0]))
\end{align*}
and 
\begin{align*}
D &\subseteq D' \subseteq \content(T[n]), \\
h(D) &\neq h(D').
\end{align*}
As $W_{h(D)} = L$ and $h$ is conservative, $h$ must not change its hypothesis. Thus, for all $D'$ with $D \subseteq D' \subseteq L$ we get $h(D') = h(D)$, a contradiction. 

Thus we have, for all $n \geq n_0$, 
\begin{align*}
h'(\content(T[n])) &= h'(\content(T[n_0])) \\
&= \pad(h(\content(T[n_0])),0)
\end{align*} 
and $W_{h'(\content(T[n_0]))} = W_{\pad(h(\content(T[n_0])),0)} = L$, i.e.\ $h'$ $\Text\Gold\Ex$-learns $L$.

Second, we will show that $h'$ is conservative. Whenever $h$ makes a mind change, $h'$ will also make a mind change; as, for all $n$, $W_{h(\content(T[n]))} = W_{h'(\content(T[n]))}$, we have that $h'$ is conservative in these cases. Thus, we have to show that $h'$ is conservative whenever it changes its mind because the first condition in the definition does not hold. Let $n$ such that $$h'(\content(T[n])) \neq h'(\content(T[n-1]))$$ because the first condition in the definition of $h'$ is violated. Let $C = \content(T[n])$.
Thus, there are $D$ and $D'$ with $D \subseteq D' \subseteq C$ such that $h(D) = h(C)$ and $h(D') \neq h(C)$. We consider the case that $h(T[n]) = h(T[n-1])$ as otherwise $h'$ is obviously conservative. As $h$ is conservative we can conclude that there is $x \in D'$ such that $x \notin W_{h(D)}$. If not we could construct a text $T'$ with elements of $D$ on which $h$ would not be conservative. Thus there is $x \in D' \subseteq C$ such that 
$$
x \notin W_{h(C)} = W_{h(T[n])} = W_{h(T[n-1])} = W_{h'(T[n-1])}
$$
and therefore $h'$ is still conservative if it changes its mind. 

To show that $h'$ is syntactically decisive let $C \subseteq D \subseteq E$ such that $h'(C) \neq h'(D)$ and $h'(C) = h'(E)$. This implies that $C \subset E$. Thus $0 \neq |C| + 1 \neq |E|+1$ and therefore the second component in $\pad$ is different for $C$ and $E$. This implies that $h'(C) \neq h'(E)$ as $\pad$ is injective.
\end{proof}

\setboolean{useproof}{true}

The following Theorem is the main result of this section, showing that set-driven learning can be done not just conservatively, but also strongly decisively and cautiously \emph{at the same time}.

\begin{thm}\label{thm:SdConvCautSDec}
We have $$[\Text\Sd\Ex] = [\Text\Sd\Conv\SDec\Caut\Ex].$$ 
\end{thm}
\begin{proof}
Following \cite{Kin-Ste:j:95:mon} we can assume a set-driven learner to be conservative.
Let $h$ and $\mathcal{L}$ be such that $h$ \textbf{TxtSdConvEx}-learns $\mathcal{L}$ and suppose that $h$ is syntactically decisive using Lemma \ref{thm:SdSyntDec}. We define a function $p$ using the S-m-n Theorem such that, for every set $D$ and $e$,
$$W_{p(D,e)} = D \bigcup_{t \in \natnum} \begin{cases} W_e^t, & \text{if } h(D \cup W_e^t) = e; \\
	\emptyset, & \text{otherwise.} \end{cases}$$
We define a function $N$ such that, for any finite set $D$,
\begin{align*}
N(D) = \{ D' \subseteq D\ |\ &h(D) = h(D')\}.\end{align*}
We define $h'$, for all finite sets $D$,  as
$$
h'(D) = p(\min(N(D)), h(D))
$$
Let $L \in \mathcal{L}$ be given and let $T$ be a text for $L$. 
We first show that $h'$ $\Text\Sd\Ex$-learns $L$ from $T$. 
As $h$ \textbf{TxtSdEx}-learns $L$ we know that $h$ is strongly locking on $T$ (this was shown in~\cite{Cas-Koe:c:10:colt}). Thus there is $n_0$ such that $T[n_0]$ is a locking sequence. Let $D' \subseteq \content(T[n_0])$ be minimal with $h(D') = h(\content(T[n_0]))$.
Thus we have, for all $n \geq n_0$, $\min(N(\content(T[n]))) = D'$. From the construction of $p$ and $h$ syntactically decisive we get 
$$W_{p(D',h(D'))} = W_{h(D')}.$$
This shows that $h'$ $\Text\Sd\Ex$-learns $L$.

Next we show the following claim. \begin{claim}\label{claim:SynDecConcl}
$\forall D\ (\forall D' \subseteq D\ |\ D' \notin N(D))\ \forall C \in N(D) : C\backslash W_{h'(D')} \neq \emptyset.$
\end{claim}
\begin{claimProof}
As $h$ is syntactically decisive we have that, for all $D''$ with $D' \subseteq D'' \subseteq D$, $h(D') = h(D'') = h(D).$ Therefore we get 
$$h(D') \neq h(D' \cup C).$$
Suppose, by way of contradiction, $C \subseteq W_{h'(D')}$. This implies that there is $t$ such that $C \subseteq D' \cup W_{h(D')}^t$ with $h(D' \cup W_{h(D')}^t) = h(D')$, according to the definitions of $h'$ and $p$. But, as $D' \subseteq D' \cup C \subseteq D' \cup W_{h(D')}^t$, this is a contradiction to $h$ being syntactically decisive.
\end{claimProof}

Let $i \leq j$ be such that $h'(\content(T[i])) \neq h'(\content(T[j]))$. To increase readability we let $D_0 = \content(T[i])$ and $D_1 = \content(T[j])$.
As $h$ is syntactically decisive, $h'$ only changes its mind if $h$ changed its mind before. Thus we have $h(D_0) \neq h(D_1).$ 
As $D_0 \subseteq D_1$ and $D_0 \notin N(D_1)$ we get from Claim~\ref{claim:SynDecConcl} (with $C= D = D_1$ and $D' = D_0$) that $$D_1 \backslash W_{h'(D_0)} \neq \emptyset.$$ 
This shows that $h'$ is conservative.
We will now show that $$W_{h'(D_1)} \nsubseteq W_{h'(D_0)},$$ as this implies that $h'$ is cautious and strongly decisive.

From the construction of $h'$ we get that there is $B \subseteq D_1$ with $h(B) = h(D_1)$ such that $h'$ is consistent on $B$, i.e.\ $B \subseteq W_{h'(D_1)}.$ Using Claim~\ref{claim:SynDecConcl} again (this time with $C = B$, $D = D_1$ and $D' = D_0$), we see that there is 
$$x \in B \backslash W_{h'(D_0)} \subseteq W_{h'(D_1)} \backslash W_{h'(D_0)},$$
which shows that $W_{h'(D_0)} \not\subseteq W_{h'(D_1)}$.
      \end{proof}

\section{Monotone Learning}
\setboolean{useproof}{true}

\label{sec:Monotone}

In this section we show the hierarchies regarding monotone and strongly monotone learning, simultaneously for the settings of $\Gold$ and $\Sd$ in Theorems~\ref{thm:SMon} and~\ref{thm:WMonNotMon}. With Theorems~\ref{thm:NatnumSDec} and~\ref{thm:MonInSDec} we establish that monotone learnabilty implies strongly decisive learnability.

\setboolean{useproof}{true}

\begin{thm}\label{thm:SMon}
There is a language $\CalL$ that is $\Text\Sd\Mon\WMon\Ex$-learnable but not $\Text\Gold\SMon\Ex$-learnable, i.e.
$$[\Text\Sd\Mon\WMon\Ex] \backslash [\Text\Gold\SMon\Ex] \neq \emptyset.$$
\end{thm}
\begin{proof} This is a standard proof which we include for completeness.
Let $L_k = \{0, 2, 4, \dots, 2k, 2k+1\}$ and $\CalL = \{2\natnum\}\cup\{L_k\ |\ k \in \natnum\}$.
Let $e$ such that $W_e = 2\natnum$ and $p$ using the S-m-n Theorem such that, for all $k$,  
$$W_{p(k)} = L_k.$$ 
We first show that $\CalL$ is $\Text\Sd\Mon\WMon\Ex$-learnable. We let a learner $h$ such that, for all $\sigma$, 
$$h(\content(\sigma)) = \begin{cases} e, & \text{if every } x \in \content(\sigma) \text{ is even;} \\ p(y), & \text{if } y \text{ is the least odd datum in } \content(\sigma). \end{cases}$$
Let $L_k \in \CalL$ and $T$ be a text for $L_k$. Thus, there is $n_0$ such that $T(n_0-1) = 2k+1$ and any element in $\content(T[n_0-1])$ is even. Then, we have, for all $n \geq n_0$, $h(\content(T[n_0])) = h(\content(T[n]))$ and $W_{h(t[n_0])} = W_{p(k)} = L_k$. It is easy to see that $h$ makes exactly one mind change on $T$ and this is at $n_0$. We have $W_e \cap \content(T)$ is a subset of $W_{p(k)} \cap \content(T)$ as $\{0,2, \dots, 2k\} \subseteq L_k$. Thus $h$ is monotone. Additionally $h$ is weakly monotone as it change its mind only if the first time a odd element is presented in the text and the previous hypotheses are $2\natnum$. \par 
Now, suppose that there is $h' \in \CalR$ and $h'$ $\Text\Gold\SMon\Ex$-learns $\CalL$. Let $\sigma$ be a locking sequence of $h'$ on $2\natnum$ and $k$ such that, for all $x \in \content(\sigma), x \leq 2k+1$. We let $T$ be a text for $L_k$ starting with $\sigma$. As $2\natnum \nsubseteq L_k$ we have that $h'$ is not strongly monotone on $T$ or $h$ does not $\Text\Gold\Ex$-learns $L_k$ from $T$. 
\end{proof}

\setboolean{useproof}{true}

\begin{thm}\label{thm:WMonNotMon}
There is $\CalL$ such that $\CalL$ is $\Text\Sd\WMon\Ex$-learnable but not $\Text\Gold\Mon\Ex$-learnable.
\end{thm}
\begin{proof} This is a standard proof which we include for completeness.
Let $L_k = \{x\ |\ x \leq 2k+1\}$ and $\CalL = \{2\natnum\} \cup \{L_k\ |\ k \in \natnum\}$.
Let $e$ such that $W_e = 2\natnum$ and $p$ using the S-m-n Theorem such that, for all $k$,  $$W_{p(k)} = L_k.$$ We define, for all $\sigma$, a learner $h$ such that
$$h(\content(\sigma)) = \begin{cases} e, & \text{if every element in } \content(\sigma) \text{ is even;} \\ p(y), &\text{else, } y \text{ is the maximal odd element in } \content(\sigma). \end{cases}$$
Let $L_k \in \CalL$ and a $T$ be a text for $L_k$. Then, there is $n_0$ such that $2k+1 \in \content(T[n_0])$ for the first time. Thus we have that for all $n \geq n_0, h(\content(T[n_0])) = h(\content(T[n]))$ and $W_{h(\content(T[n_0]))} = W_{p(k)} = L_k$. Obviously $h$ learns $L_k$ weakly mononote as the learner only change its mind if a greater odd element appears in the text. \par 
Suppose now there is a learner $h' \in \CalR$ such that $h'$ $\Text\Gold\Mon\Ex$-learns $\CalL$. Let $\sigma$ be a locking sequence of $h'$ on $2\natnum$ and $k$ such that, for all $x \in \content(\sigma)$, $x \leq 2k+1$. Let $\sigma' \supseteq \sigma$ a locking sequence of $h'$ on $L_{k}$ and $T$ be a text for $L_{k+1}$ starting with $\sigma'$. Let $\sigma'' \supseteq \sigma'$ be a locking sequence of $h'$ on $L_{k+1}$. Then, we have
\begin{align*}
W_{h'(\sigma)} &= 2\natnum; \\
W_{h'(\sigma')} &= L_k ; \\
W_{h'(\sigma'')} &= L_{k+1}.
\end{align*}
As the datum $2k+2$ is in $2\natnum$ and in $L_{k+1}$ but not in $L_k$, $h'$ is not monotone on the text $T$ for $L_{k+1}$. 
\end{proof}

\setboolean{useproof}{true}

The following theorem is an extension of a theorem from~\cite{Bal-Cas-Mer-Ste-Wie:j:08}, where the theorem has been shown for decisive learning instead of strongly decisive learning.

\begin{thm}\label{thm:NatnumSDec}
Let $\natnum \in \CalL$ and $\CalL$ be $\Text\Gold\Ex$-learnable. Then, we have $\CalL$ is $\Text\Gold\SDec\Ex$-learnable.
\end{thm}
\begin{proof}
Let $h$ be a learner in Fulk normal form such that $h$ $\Text\Gold\Ex$-learns $\CalL$ with $\natnum \in \CalL$. As $h$ is strongly locking on $\CalL$ there is a locking sequence of $h$ on $\natnum$. Using this locking sequence we get an uniformly enumerable sequence $(L_i)_{i\in \natnum}$ of languages such that,
\begin{enumerate}
\item for $i \neq j$ and $L \supseteq L_i$, $L' \supseteq L_j$ with $L_i =^* L$, $L_j =^* L'$, $L \neq L'$;
\item for all $L \supseteq L_i$ with $L_i =^* L$, $L \notin \CalL$.
\end{enumerate}
We define a set $N(\sigma)$ such that, for every $\sigma$,
$$N(\sigma) = L_{|\sigma|} \cup \content(\sigma).$$

We define, for all $\sigma$, a set $M(\sigma)$ such that
$$M(\sigma) = \{\lambda
\} \cup \{\tau\ |\ \tau \subseteq \sigma\ \land\ h(\tau) \neq h(\tau^-)\ \land\ \forall x \in \content(\tau) : \Phi_{h(\tau)}(x) \leq \left|\sigma\right| \}.$$ 
Using the S-m-n Theorem we get a function $p \in \CalR$ such that, for all $\sigma$,
$$
W_{p(\sigma)} = \bigcup_{t \in \mathbb{N}} 
\begin{cases} 
W_{h(\sigma)}^t, 			& \text{if } \forall \rho \in W_{h(\sigma)}^t : h(\sigma) = h(\sigma \diamond \rho); \\
N(\sigma), 					& \text{otherwise.} 
\end{cases}
$$

We will use the $p(\sigma)$ as hypotheses. Note that any hypothesis $p(\sigma)$ is either semantically equivalent to $h(\sigma)$ or, if $\sigma$ is not a locking sequence of $h$ for any language, $p(\sigma)$ is an index for a finite superset of $L_{\sigma}$. In the latter case we call the hypothesis $p(\sigma)$ \emph{poisoned}.

We define a learner $h'$ such that, for all $\sigma$, 
$$h'(\sigma) = p(\max(M(\sigma))).$$

Let $L \in \CalL$ and $T$ be a text for $L$. As $h$ is strongly locking and $h$ $\Text\Gold\Ex$-learns $\CalL$ there is $n_0$ such that, for all $\sigma \in \seq(L)$, $h(T[n_0]) = h(T[n_0] \diamond \sigma)$ and $W_{h(T[n_0])} = L$. Thus, there is $n_1 > n_0$ such that, for all $x \in \content(T[n_0])$, $\Phi_{h(T[n_0])}(x) \leq n_1$. This implies that, for all $n \geq n_1$, $h'(T[n_1]) = h'(T[n])$ and 
$$W_{h'(T[n_1])} = W_{p(\max(M(T[n_1])))} = \bigcup_{t \in \mathbb{N}} W_{h(T[n_0])}^t = L.$$

Next, we will show that $h'$ is strongly decisive. Suppose there are $i \leq j \leq k$ such that $W_{h'(T[i])} = W_{h'(T[k])}$ and $h'(T[i]) \neq h'(T[j])$. From the construction of the learner $h'$ we get $h(T[i]) \neq h(T[j])$.

\textit{Case 1:} $h'(T[i])$ is \emph{not} a poisoned hypothesis.
Independently of whether $h'(T[k])$ is poisoned or not, there is $\sigma \subseteq T[k]$ such that $\content(\sigma) \subseteq W_{h'(T[k])}$. ($T[k]$ if the hypothesis is poisend, $\max(M(T[k]))$ otherwise.) As $h'(T[i])$ is not poisened and $h(T[i]) \neq h(T[k])$ we get through the  construction of $p$ that $\content(\sigma) \nsubseteq W_{h'(T[i])}$. 
Thus, we have $W_{h'(T[i])} \neq W_{h'(T[k])}$, a contradiction.

\textit{Case 2:} $h'(T[i])$ \emph{is} poisoned. Thus, we have $T[i] \subseteq W_{h'(T[i])}$.

\textit{Case 2.1:} $h'(T[k])$ is \emph{not} poisoned. Thus, $T[k]$ is a locking sequence on $h$ for a language $L \in \Text\Gold\Ex(h)$ and $W_{h'(T[k])} \in \Text\Gold\Ex(h)$. As $h'(T[i])$ is poisoned we have $W_{h'(T[i])} \notin \Text\Gold\Ex(h)$. Thus, we get $W_{h'(T[i])} \neq W_{h'(T[k])}$, a contradiction.

\textit{Case 2.2:} $h'(T[k])$ \emph{is} poisoned. As $T[i] \subset T[k]$ and $N(T[i]) =^* W_{h'(T[i])}$ and $N(T[k]) =^* W_{h'(T[k])}$ we have $W_{h'(T[i])} \neq W_{h'(T[k])}$.

\end{proof}

\setboolean{useproof}{true}

\begin{thm}\label{thm:MonInSDec}
We have that any monotone $\Text\Gold\Ex$-learnable class of languages is strongly decisive learnable, while the converse does not hold, i.e.
$$[\Text\Gold\Mon\Ex] \subset [\Text\Gold\SDec\Ex].$$
\end{thm}
\begin{proof}
Let $h \in \CalR$ be a learner and $\CalL = \Text\Gold\Mon\Ex(h)$. We distinguish the following two cases. We call $\CalL$ \emph{dense} iff it contains a superset of every finite set.

\textit{Case 1:} $\CalL$ is dense. We will show now that $h$ $\Text\Gold\SMon\Ex$-learns the class $\CalL$. Let $L \in \CalL$ and $T$ be a text for $L$. Suppose there are $i$ and $j$ with $i < j$ such that $W_{h(T[i])} \nsubseteq W_{h(T[j])}$. Thus, we have $W_{h(T[i])}\backslash W_{h(T[j])} \neq \emptyset$. Let $x \in W_{h(T[i])}\backslash W_{h(T[j])}$. As $\CalL$ is dense there is a language $L' \in \CalL$ such that $\content(T[j]) \cup \{x\} \in L'$. Let $T'$ be a text for $L'$ and $T''$ be such that $T'' = T[j] \diamond T'$. Obviously, $T''$ is a text for $L'$. 
We have that $x \in W_{h(T''[i])}$ but $x \notin W_{h(T''[j])}$ which is a contradiction as $h$ is monotone. 
Thus, $h$ $\Text\Gold\SMon\Ex$-learns $\CalL$, which implies that $h$ $\Text\Gold\WMon\Ex$-learns $\CalL$. Using Corollary~\ref{cor:WMonInSDec} we get that $\CalL$ is $\Text\Gold\SDec\Ex$-learnable.

\textit{Case 2:} $\CalL$ is not dense. Thus, $\CalL' = \CalL \cup \natnum$ is $\Text\Gold\Ex$-learnable. Using Theorem~\ref{thm:NatnumSDec} $\CalL'$ is $\Text\Gold\SDec\Ex$-learnable and therefore so is $\CalL$.

Note that $[\Text\Gold\SDec\Ex] \subseteq [\Text\Gold\Mon\Ex]$ does not hold as in \textit{Case 1} with Corollary~\ref{cor:WMonInSDec} a proper subset relation is used.

\end{proof}

\bibliographystyle{alpha}

\begin{thebibliography}{BCM{\etalchar{+}}08}

\bibitem[Ang80]{Ang:j:80:lang-pos-data}
D.~Angluin.
\newblock Inductive inference of formal languages from positive data.
\newblock {\em Information and Control}, 45:117--135, 1980.

\bibitem[BB75]{Blu-Blu:j:75}
L.~Blum and M.~Blum.
\newblock Toward a mathematical theory of inductive inference.
\newblock {\em Information and Control}, 28:125--155, 1975.

\bibitem[BCM{\etalchar{+}}08]{Bal-Cas-Mer-Ste-Wie:j:08}
G.~Baliga, J.~Case, W.~Merkle, F.~Stephan, and W.~Wiehagen.
\newblock When unlearning helps.
\newblock {\em Information and Computation}, 206:694--709, 2008.

\bibitem[CK10]{Cas-Koe:c:10:colt}
J.~Case and T.~K{\"o}tzing.
\newblock Strongly non-{U}-shaped learning results by general techniques.
\newblock In {\em Proc.~of COLT (Conference on Learning Theory)}, pages
  181--193, 2010.

\bibitem[CM11]{Cas-Moe:j:11:optLan}
J.~Case and S.~Moelius.
\newblock Optimal language learning from positive data.
\newblock {\em Information and Computation}, 209:1293--1311, 2011.

\bibitem[Ful90]{Ful:j:90:prudence}
M.~Fulk.
\newblock Prudence and other conditions on formal language learning.
\newblock {\em Information and Computation}, 85:1--11, 1990.

\bibitem[Gol67]{Gol:j:67}
E.~Gold.
\newblock Language identification in the limit.
\newblock {\em Information and Control}, 10:447--474, 1967.

\bibitem[Jan91]{Jan:j:91}
K.~Jantke.
\newblock Monotonic and non-monotonic inductive inference of functions and
  patterns.
\newblock In J.~Dix, K.~Jantke, and P.~Schmitt, editors, {\em Nonmonotonic and
  Inductive Logic}, volume 543 of {\em Lecture Notes in Computer Science},
  pages 161--177. 1991.

\bibitem[JORS99]{Jai-Osh-Roy-Sha:b:99:stl2}
S.~Jain, D.~Osherson, J.~Royer, and A.~Sharma.
\newblock {\em Systems that Learn: {A}n Introduction to Learning Theory}.
\newblock MIT Press, Cambridge, Massachusetts, second edition, 1999.

\bibitem[JS98]{Jai-Sha:j:98}
S.~Jain and A.~Sharma.
\newblock Generalization and specialization strategies for learning r.e.
  languages.
\newblock {\em Annals of Mathematics and Artificial Intelligence}, 23:1--26,
  1998.

\bibitem[K{\"o}t09]{Koe:th:09}
T.~K{\"o}tzing.
\newblock {\em Abstraction and Complexity in Computational Learning in the
  Limit}.
\newblock PhD thesis, University of Delaware, 2009.
\newblock Available online at\\
  http://pqdtopen.proquest.com/\#viewpdf?dispub=3373055.

\bibitem[K{\"o}t14]{Koe:c:14:stacs}
T.~K{\"o}tzing.
\newblock A solution to {W}iehagen's thesis.
\newblock In {\em Proc.~of STACS (Symposium on Theoretical Aspects of Computer
  Science)}, pages 494--505, 2014.

\bibitem[KS95]{Kin-Ste:j:95:mon}
E.~Kinber and F.~Stephan.
\newblock Language learning from texts: Mind changes, limited memory and
  monotonicity.
\newblock {\em Information and Computation}, 123:224--241, 1995.

\bibitem[LZ93]{Lan-Zeu:c:93}
S.~Lange and T.~Zeugmann.
\newblock Monotonic versus non-monotonic language learning.
\newblock In {\em Proc.~of Nonmonotonic and Inductive Logic}, pages 254--269,
  1993.

\bibitem[OSW82]{Osh-Sto-Wei:j:82:strategies}
D.~Osherson, M.~Stob, and S.~Weinstein.
\newblock Learning strategies.
\newblock {\em Information and Control}, 53:32--51, 1982.

\bibitem[OSW86]{Osh-Sto-Wei:b:86:stl}
D.~Osherson, M.~Stob, and S.~Weinstein.
\newblock {\em Systems that Learn: {A}n Introduction to Learning Theory for
  Cognitive and Computer Scientists}.
\newblock MIT Press, Cambridge, Mass., 1986.

\bibitem[Rog67]{Rog:b:87}
H.~Rogers.
\newblock {\em Theory of Recursive Functions and Effective Computability}.
\newblock McGraw Hill, New York, 1967.
\newblock Reprinted by MIT Press, Cambridge, Massachusetts, 1987.

\bibitem[SR84]{Sch:th:84}
G.~Sch\"afer-Richter.
\newblock {\em \"Uber Eingabeabh\"angigkeit und Komplexit\"at von
  Inferenzstrategien}.
\newblock PhD thesis, RWTH Aachen, 1984.

\bibitem[WC80]{Wex-Cul:b:80}
K.~Wexler and P.~Culicover.
\newblock {\em Formal Principles of Language Acquisition}.
\newblock MIT Press, Cambridge, Massachusetts, 1980.

\bibitem[Wie91]{Wie:c:91}
R.~Wiehagen.
\newblock A thesis in inductive inference.
\newblock In {\em Proc.~of Nonmonotonic and Inductive Logic}, pages 184--207,
  1991.

\end{thebibliography}

\newcommand{\etalchar}[1]{$^{#1}$}

\end{document}